\documentclass{article} 
\usepackage{iclr2021_conference,times}

\usepackage{hyperref}
\usepackage{url}

\usepackage{booktabs}       
\usepackage{amsfonts}       
\usepackage{nicefrac}       
\usepackage{microtype}      
\usepackage{amsmath}
\usepackage{amsthm}
\usepackage{subcaption}
\usepackage{algorithm}
\usepackage[noend]{algpseudocode}
\usepackage{mathtools}

\usepackage{graphicx}
\usepackage{tikz}

\newtheorem{theorem}{Theorem}
\newtheorem{lemma}{Lemma}
\newtheorem{corollary}{Corollary}

\newtheorem{definition}{Definition}
\newtheorem{assumption}{Assumption}

\DeclareMathOperator*{\maximize}{maximize}
\DeclareMathOperator*{\argmax}{arg\,max}

\DeclareMathOperator{\residual}{Residual}
\DeclareMathOperator{\fc}{FC}
\DeclareMathOperator{\conv}{Conv}
\DeclareMathOperator{\concat}{Concat}
\DeclareMathOperator{\networkinput}{Input}

\title{Learning perturbation sets for robust machine learning}

\author{%
  Eric Wong \\
  Department of Electrical Engineering and Computer Science
\\
  Massachusetts Institute of Technology\\
  Cambridge, MA 02139, USA \\
  \texttt{wongeric@mit.edu} \\
   \And
   J. Zico Kolter \\
   Computer Science Department \\
   Carnegie Mellon University \\
   Pittsburgh, PA 15213, USA \\
   \texttt{zkolter@cs.cmu.edu} \\
}

\iclrfinalcopy 
\begin{document}

\maketitle

\begin{abstract}
Although much progress has been made towards robust deep learning, a significant gap in robustness remains between real-world perturbations and more narrowly defined sets typically studied in adversarial defenses. In this paper, we aim to bridge this gap by \emph{learning} perturbation sets from data, in order to characterize real-world effects for robust training and evaluation. Specifically, we use a conditional generator that defines the perturbation set over a constrained region of the latent space. We formulate desirable properties that measure the quality of a learned perturbation set, and theoretically prove that a conditional variational autoencoder naturally satisfies these criteria. Using this framework, our approach can generate a variety of perturbations at different complexities and scales, ranging from baseline spatial transformations, through common image corruptions, to lighting variations. We measure the quality of our learned perturbation sets both quantitatively and qualitatively, finding that our models are capable of producing a diverse set of meaningful perturbations beyond the limited data seen during training. Finally, we leverage our learned perturbation sets to train models which are empirically and certifiably robust to adversarial image corruptions and adversarial lighting variations, while improving generalization on non-adversarial data. All code and configuration files for reproducing the experiments as well as pretrained model weights can be found at \url{https://github.com/locuslab/perturbation_learning}. 
\end{abstract}
\section{Introduction}
Within the last decade, adversarial learning has become a core research area for studying robustness and machine learning. Adversarial attacks have expanded well beyond the original setting of imperceptible noise to more general notions of robustness, and can broadly be described as capturing sets of perturbations that humans are naturally invariant to. These invariants, such as facial recognition should be robust to adversarial glasses \citep{sharif2019general} or traffic sign classification should be robust to adversarial graffiti \citep{eykholt2018robust}, form the motivation behind many real world adversarial attacks. However, human invariants can also include notions which are not inherently adversarial, for example image classifiers should be robust to common image corruptions \citep{hendrycks2019benchmarking} as well as changes in weather patterns \citep{michaelis2019benchmarking}. 

On the other hand, although there has been much success in defending against small adversarial perturbations, most successful and principled methods for learning robust models are limited to human invariants that can be characterized using mathematically defined perturbation sets, for example perturbations bounded in $\ell_p$ norm. After all, established guidelines for evaluating adversarial robustness \citep{carlini2019evaluating} have emphasized the importance of the perturbation set (or the threat model) as a necessary component for performing proper, scientific evaluations of adversarial defense proposals. However, this requirement makes it difficult to learn models which are robust to human invariants beyond these mathematical sets, where real world attacks and general notions of robustness can often be virtually impossible to write down as a formal set of equations.  This incompatibility between existing methods for learning robust models and real-world, human invariants raises a fundamental question for the field of robust machine learning: 

\begin{center}
\emph{How can we learn models that are robust to perturbations without a predefined perturbation set?}
\end{center}

In the absence of a mathematical definition, in this work we present a general framework for learning perturbation sets from perturbed data. 
More concretely, given pairs of examples where one is a perturbed version of the other, we propose learning generative models that can ``perturb'' an example by varying a fixed region of the underlying latent space. The resulting perturbation sets are well-defined and can naturally be used in robust training and evaluation tasks. The approach is widely applicable to a range of robustness settings, as we make no assumptions on the type of perturbation being learned: the only requirement is to collect pairs of perturbed examples. 

Given the susceptibility of deep learning to adversarial examples, such a perturbation set will undoubtedly come under intense scrutiny, especially if it is to be used as a threat model for adversarial attacks. In this paper, we begin our theoretical contributions with a broad discussion of perturbation sets and formulate deterministic and probabilistic properties that a learned perturbation set should have in order to be a meaningful proxy for the true underlying perturbation set. The \emph{necessary subset} property ensures that the set captures real perturbations, properly motivating its usage as an adversarial threat model. The \emph{sufficient likelihood} property ensures that real perturbations have high probability, which motivates sampling from a perturbation set as a form of data augmentation. 
We then prove the main theoretical result, that a learned perturbation set defined by the decoder and prior of a conditional variational autoencoder (CVAE) \citep{sohn2015learning} implies both of these properties, providing a theoretically grounded framework for learning perturbation sets. 
The resulting CVAE perturbation sets are well motivated, can leverage standard architectures, and are computationally efficient with little tuning required.



We highlight the versatility of our approach using CVAEs with an array of experiments, where we vary the complexity and scale of the datasets, perturbations, and downstream tasks. We first demonstrate how the approach can learn basic $\ell_\infty$ and rotation-translation-skew (RTS) perturbations \citep{jaderberg2015spatial} in the MNIST setting. Since these sets can be mathematically defined, our goal is simply to measure exactly how well the learned perturbation set captures the target perturbation set on baseline tasks where the ground truth is known. We next look at a more difficult setting which can not be mathematically defined, and learn a perturbation set for common image corruptions on CIFAR10 \citep{hendrycks2019benchmarking}. The resulting perturbation set can interpolate between common corruptions, produce diverse samples, and be used in adversarial training and randomized smoothing frameworks. The adversarially trained models have improved generalization performance to both in- and out-of-distribution corruptions and better robustness to adversarial corruptions. In our final setting, we learn a perturbation set that captures \emph{real-world} variations in lighting using a multi-illumination dataset of scenes captured ``in the wild'' \citep{murmann2019dataset}. The perturbation set generates meaningful lighting samples and interpolations while generalizing to unseen scenes, and can be used to learn image segmentation models that are empirically and certifiably robust to lighting changes. 
All code and configuration files for reproducing the experiments as well as pretrained model weights for both the learned perturbation sets as well as the downstream robust classifiers are at \url{https://github.com/locuslab/perturbation_learning}. 





\section{Background and related work}

\paragraph{Perturbation sets for adversarial threat models}
Adversarial examples were initially defined as imperceptible examples with small $\ell_1$, $\ell_2$ and $\ell_\infty$ norm \citep{biggio2013evasion, szegedy2013intriguing, goodfellow2014explaining}, forming the earliest known, well-defined perturbation sets that were eventually generalized to the union of multiple $\ell_p$ perturbations \citep{tramer2019adversarial, maini2019adversarial, croce2019provable, stutz2019confidence}. Alternative perturbation sets to the $\ell_p$ setting that remain well-defined incorporate more structure and semantic meaning, such as rotations and translations \citep{engstrom2017rotation}, Wasserstein balls \citep{wong2019wasserstein}, functional perturbations \citep{laidlaw2019functional}, distributional shifts \citep{sinha2017certifying, sagawa2019distributionally}, word embeddings \citep{miyato2016adversarial}, and word substitutions \citep{alzantot2018generating, jia2019certified}. 

Other work has studied perturbation sets that are not necessarily mathematically formulated but well-defined from a human perspective such as spatial transformations \citep{xiao2018spatially}. Real-world adversarial attacks tend to try to remain either inconspicuous to the viewer or meddle with features that humans would naturally ignore, such as textures on 3D printed objects \citep{athalye2017synthesizing}, graffiti on traffic signs \citep{eykholt2018robust}, shapes of objects to avoid LiDAR detection \citep{cao2019adversarial}, irrelevant background noise for audio \citep{li2019adversarialmusic}, or barely noticeable films on cameras \citep{li2019adversarial}. 
Although not necessarily adversarial, \citet{hendrycks2019benchmarking} propose the set of common image corruptions as a measure of robustness to informal shifts in distribution. 

\paragraph{Generative modeling and adversarial robustness}
Relevant to our work is that which combines aspects of generative modeling with adversarial examples. While our work aims to learn \emph{real-world} perturbation sets from data, most work in this space differs in that they either aim to generate synthetic adversarial $\ell_p$ perturbations \citep{xiao2018generating}, run user studies to define the perturbation set \citep{sharif2019general}, or simply do not restrict the adversary at all \citep{song2018constructing, bhattad2020unrestricted}. 
\citet{gowal2019achieving} trained a StyleGAN to disentangle real-world perturbations when no perturbation information is known in advance. However the resulting perturbation set relies on a stochastic approximation, and it is not immediately obvious what this set will ultimately capture. Most similar is the concurrent work of \citet{robey2020model}, which uses a GAN architecture from image-to-image translation to model simple perturbations between datasets. In contrast to both of these works, our setting requires the collection of paired data to directly learn \emph{how} to perturb from perturbed pairs without needing to disentangle any features or translate datasets, allowing us to learn more targeted and complex perturbation sets. Furthermore, we formulate desirable properties of perturbation sets for downstream robustness tasks, and formally prove that a conditional variational autoencoder approach satisfies these properties. This results in a principled framework for learning perturbation sets that is quite distinct from these GAN-based approaches in both setting and motivation.

\paragraph{Adversarial defenses and data augmentation}
Successful approaches for learning adversarially robust networks include methods which are both empirically robust via adversarial training \citep{goodfellow2014explaining,kurakin2016adversarial, madry2017towards} and also certifiably robust via provable bounds \citep{wong2017provable, wong2018scaling, raghunathan2018certified, gowal2018effectiveness, zhang2019towards} and randomized smoothing \citep{cohen2019certified, yang2020randomized}. Critically, these defenses require mathematically-defined perturbation sets, which has limited these approaches from learning robustness to more general, real-world perturbations. We directly build upon these approaches by learning perturbation sets that can be naturally and directly incorporated into robust training, greatly expanding the scope of adversarial defenses to new contexts. 
Our work also relates to using non-adversarial perturbations via data augmentation to reduce generalization error \citep{zhang2017mixup, devries2017improved, cubuk2019autoaugment}, which can occasionally also improve robustness to unrelated image corruptions \citep{geirhos2018imagenet, hendrycks2019augmix, rusak2020increasing}. Our work differs in that rather than aggregating or proposing generic data augmentations, our perturbation sets can provide data augmentation that is targeted for a particular robustness setting.  

\section{Perturbation sets learned from data}
For an example $x \in \mathbb R^m$, a perturbation set $\mathcal S(x)\subseteq \mathbb R^m$ is defined informally as the set of examples which are considered to be equivalent to $x$, and hence can be viewed as ``perturbations'' of $x$. This set is often used when finding an adversarial example, which is typically cast as an optimization problem to maximize the loss of a model over the perturbation set in order to break the model. For example, for a classifier $h$, loss function $\ell$, and label $y$, an adversarial attack tries to solve the following: 
\begin{equation}
\maximize_{x' \in \mathcal S(x)} \ell(h(x'), y).
\end{equation}
A common choice for $\mathcal S(x)$ is an $\ell_p$ ball around the unperturbed example, defined as $\mathcal S(x) = \{ x + \delta : \|\delta\|_p \leq \epsilon\}$ for some norm $p$ and radius $\epsilon$. This type of perturbation captures unstructured random noise, and is typically taken with respect to $\ell_p$ norms for $p \in \{0, 1, 2, \infty\}$, though more general distance metrics can also be used. 

Although defining the perturbation set is critical for developing adversarial defenses, in some scenarios, the \emph{true} perturbation set may be difficult to mathematically describe. 
In these settings, it may still be possible to collect observations of (non-adversarial) perturbations, e.g. pairs of examples $(x,\tilde x)$ where $\tilde x$ is the \emph{perturbed data}. In other words, $\tilde x$ is a perturbed version of $x$, from which we can learn an approximation of the true perturbation set.  
While there are numerous possible approaches one can take to learn $\mathcal S(x)$ from examples $(x, \tilde x$), in this work we take a generative modeling perspective, where examples are perturbed via an underlying latent space. Specifically, let $g: \mathbb R^k \times \mathbb R^m \rightarrow \mathbb R^m$ be a generator that takes a $k$-dimensional latent vector and an input, and outputs a perturbed version of the input. Then, we can define a \emph{learned} perturbation set as follows: 
\begin{equation}
\label{eq:perturbation_set}
\mathcal S(x) = \{ g(z,x) : \|z\| \leq \epsilon\}
\end{equation}
In other words, we have taken a well-defined norm-bounded ball in the latent space and mapped it to a set of perturbations with a generator $g$, which perturbs $x$ into $\tilde x$ via a latent code $z$. 
Alternatively, we can define a perturbation set from a \emph{probabilistic modeling} perspective, and use a distribution over the latent space to parameterize a distribution over examples. Then, $\mathcal S(x)$ is now a random variable defined by a probability distribution $p_\epsilon(z)$ over the latent space as follows: 
\begin{equation}
\label{eq:probabilistic_perturbation_set}
\mathcal S(x) \sim p_{\theta}\;\; \textrm{such that}\; \theta = g(z,x), \quad z \sim p_\epsilon
\end{equation} 
where $p_\epsilon$ has support $\{z : \|z\| \leq \epsilon\}$ and $p_{\theta}$ is a distribution parameterized by $\theta=g(z,x)$. 


\subsection{General measures of quality for perturbation sets}
\label{sec:properties}
A perturbation set defined by a generative model that is learned from data lacks the mathematical rigor of previous sets, so care must be taken to properly evaluate how well the model captures real perturbations. In this section we formally define two properties relating a perturbation set to data, which capture natural qualities of a perturbation set that are useful for adversarial robustness and data augmentation. 
We note that all quantities discussed in this paper can be calculated on both the training and testing sets, which allow us to concretely measure how well the perturbation set generalizes to unseen datapoints. 
For this section, let $d:\mathbb R^m \times\mathbb R^m\rightarrow \mathbb R$ be an distance metric (e.g. mean squared error) and let $x, \tilde x \in \mathbb{R}^m$ be a perturbed pair, where $\tilde x$ is a perturbed version of $x$. 

To be a reasonable threat model for adversarial examples, one desirable expectation is that a perturbation set should at least contain close approximations of the perturbed data. In other words, the set of perturbed data should be (approximately) a \emph{necessary subset} of the perturbation set. This notion of containment can be described more formally as follows: 
\begin{definition}
A perturbation set $\mathcal S(x)$ satisfies the \emph{necessary subset property} at approximation error at most $\delta$ for a perturbed pair $(x, \tilde x)$ if there exists an $x'\in \mathcal S(x)$ such that $d( x', \tilde x) \leq \delta$. 
\end{definition}
For a perturbation set defined by the generative model from Equation \eqref{eq:perturbation_set}, this amounts to finding a latent vector $z$ which best approximates the perturbed example $\tilde x$ by solving the following problem: 
\begin{equation}
\label{eq:necessary_subset}
\min_{\|z\|\leq \epsilon} d(g(z,x),\tilde x). 
\end{equation}
This approximation error can be upper bounded with point estimates or can be solved more accurately with projected gradient descent. Note that mathematically-defined perturbation sets such as $\ell_p$ balls around clean datapoints contain all possible observations and naturally have zero approximation error. 

Our second desirable property is specific to the probabilistic view from Equation \eqref{eq:probabilistic_perturbation_set}, where we would expect perturbed data to have a high probability of occurring under a probabilistic perturbation set. In other words, a perturbation set should assign \emph{sufficient likelihood} to perturbed data, described more formally in the following definition:  
\begin{definition}
A probabilistic perturbation set $\mathcal S(x)$ satisfies the \emph{sufficient likelihood property} at likelihood at least $\delta$ for a perturbed pair $(x,\tilde x)$ if $\mathbb E_{p_\epsilon(z)}[p_\theta(\tilde x)] \geq \delta$ where $\theta = g(z,x)$. 
\end{definition}
A model that assigns high likelihood to perturbed observations is likely to generate meaningful samples, which can then be used as a form of data augmentation in settings that care more about average-case over worst-case robustness. To measure this property, the likelihood can be approximated with a standard Monte Carlo estimate by sampling from the prior $p_\epsilon$. 

\section{Variational autoencoders for learning perturbations sets}
In this section we will focus on one possible approach using conditional variational autoencoders (CVAEs) to learn the perturbation set \citep{sohn2015learning}. We shift notation here to be consistent with the CVAE literature and consider a standard CVAE trained to generate $x\in \mathbb R^m$ from a latent space $z\in \mathbb R^k$ conditioned on some auxiliary variable $y$, which is traditionally taken to be a label. In our setting, the auxiliary variable $y$ is instead another datapoint such that $x$ is a perturbed version of $y$, but the theory we present is agnostic to the choice in auxiliary variable. Let the posterior distribution $q(z|x,y)$, prior distribution $p(z|y)$, and likelihood function $p(x|z,y)$ be the following multivariate normal distributions with diagonal variance: 
\begin{equation}
q(z|x,y) \sim \mathcal N(\mu(x,y), \sigma^2(x,y)), \quad p(z|y) \sim \mathcal N(\mu(y), \sigma^2(y)), \quad p(x|z,y) \sim \mathcal N(g(z,y), I)
\end{equation}
where $\mu(x,y)$, $\sigma^2(x,y)$, $\mu(y)$, $\sigma^2(y)$, and $g(z,y)$ are arbitrary functions representing the respective encoder, prior, and decoder networks. 
CVAEs are trained by maximizing a likelihood lower bound 
\begin{equation}
\log p(x|y) \geq \mathbb E_{q(z|x,y)}[\log p(x|z,y)] - KL(q(z|x,y) \| p(z|y))
\end{equation}
also known as the SGVB estimator, where $KL(\cdot\|\cdot)$ is the KL divergence. 
The CVAE framework lends to a natural perturbation set by simply restricting the latent space to an $\ell_2$ ball that is scaled and shifted by the prior network. 
For convenience, we will define the perturbation set in the latent space \emph{before} the reparameterization trick, so the latent perturbation set for all examples is a standard $\ell_2$ ball $\{ u : \|u\|_2 \leq \epsilon\}$ where $z = u\cdot \sigma(y) + \mu(y)$. Similarly, a probabilistic perturbation set can be defined by simply truncating the prior distribution at radius $\epsilon$ (also before the reparameterization trick). 

\subsection{Theoretical motivation of using CVAEs to learn perturbation sets}
Our theoretical results prove that optimizing the CVAE objective naturally results in both the necessary subset and sufficient likelihood properties outlined in Section \ref{sec:properties}, which motivates why the CVAE is a reasonable framework for learning perturbation sets. Note that these results are not immediately obvious, since the likelihood of the CVAE objective is taken over the full posterior while the perturbation set is defined over a constrained latent subspace determined by the prior. The proofs rely heavily on the multivariate normal parameterizations, with requiring several supporting results which relate the posterior and prior distributions. 
We give a concise, informal presentation of the main theoretical results in this section, deferring the full details, proofs, and supporting results to Appendix \ref{app:proofs}. Our results are based on the minimal assumption that the CVAE objective has been trained to some threshold as described in Assumption \ref{ass:objective}. 
\begin{assumption}
\label{ass:objective}
The CVAE objective has been trained to some thresholds $R,K_i$ as follows 
$$\mathbb E_{q(z|x,y)}[\log p(x|z,y)] \geq R, \quad KL(q(z|x,y) \| p(z|y)) \leq \frac{1}{2}\sum_{i=1}^k K_i$$ 
where each $K_i$ bounds the KL-divergence of the $i$th dimension. 
\end{assumption}
Our first theorem, Theorem \ref{thm:cvae_short}, states that the approximation error of a perturbed example is bounded by the components of the CVAE objective. The implication here is that with enough representational capacity to optimize the objective, one can satisfy the necessary subset property by training a CVAE, effectively capturing perturbed data at low approximation error in the resulting perturbation set. 
\begin{theorem}
\label{thm:cvae_short}
Let $r$ be the Mahalanobis distance which captures $1-\alpha$ of the probability mass for a $k$-dimensional standard multivariate normal for some $0 < \alpha < 1$. Then, there exists a $z$ such that $\left\|\frac{z - \mu(y)}{\sigma(y)}\right\|_2 \leq \epsilon$ and $\|g(z,y) - x\|^2_2 \leq \delta$ for 
$$\epsilon = Br + \sqrt{\sum_i K_i}, \quad \delta = -\frac{1}{1-\alpha}\left(2R + m\log(2\pi)\right)$$ 
where $B$ is a constant dependent on $K_i$. Moreover, as $R\rightarrow -\frac{1}{2}m\log(2\pi)$ and $K_i \rightarrow 0$ (the theoretical limits of these bounds\footnote{In practice, VAE architectures in general have a non-trivial gap from the approximating posterior which may make these theoretical limits unattainable.}), then $\epsilon \rightarrow r$ and $\delta \rightarrow 0$. 
\end{theorem} 
Our second theorem, Theorem \ref{thm:cvaesuperset_short}, states that the expected approximation error over the truncated prior can also be bounded by components of the CVAE objective. Since the generator $g$ parameterizes a multivariate normal with identity covariance, an upper bound on the expected reconstruction error implies a lower bound on the likelihood. This implies that one can also satisfy the sufficient likelihood property by training a CVAE, effectively learning a probabilistic perturbation set that assigns high likelihood to perturbed data.
\begin{theorem}
\label{thm:cvaesuperset_short}
Let $r$ be the Mahalanobis distance which captures $1-\alpha$ of the probability mass for a $k$-dimensional standard multivariate normal for some $0 < \alpha < 1$. Then, the \emph{truncated expected approximation error} can be bounded with 
$$\mathbb E_{p_r(u)}\left[\|g(u\cdot\sigma(y) + \mu(y),y) - x\|_2^2\right] \leq - \frac{1}{1-\alpha}(2R + m \log(2\pi))H$$
where $p_r(u)$ is a multivariate normal that has been truncated to radius $r$ and $H$ is a constant that depends exponentially on $K_i$ and $r$.
\end{theorem}
The main takeaway from these two theorems is that optimizing the CVAE objective naturally results in a learned perturbation set which satisfies the necessary subset and sufficient likelihood properties. The learned perturbation set is consequently useful for adversarial robustness since the necessary subset property implies that the perturbation set does not ``miss'' perturbed data. It is also useful for data augmentation since the sufficient likelihood property ensures that perturbed data occurs with high probability. We leave further discussion of these two theorems to Appendix \ref{app:proof_discussion}.

\begin{table}[t]
  \caption{Condensed evaluation of CVAE perturbation sets trained to produce rotation, translation, ans skew transformations on MNIST (MNIST-RTS), CIFAR10 common corruptions (CIFAR10-C) and multi-illumination perturbations (MI). The approximation error measures the necessary subset property, and the expected approximation error measures the sufficient likelihood property. }
  \label{table:all_evaluate}
  \centering
  \begin{tabular}{lrrrr}
    \toprule
    Setting                    & Approx. error & Expected approx. error & CVAE Recon. error & KL \\
    \midrule
    MNIST-RTS & $0.11$ & $0.54$ & $0.04$ & $22.2$ \\
    CIFAR10-C &  $0.005$ & $0.029$ & $0.001$ & $69.3$ \\
    MI & $0.006$ & $0.049$ & $0.004$ & $65.8$ \\
    \bottomrule
  \end{tabular}
\end{table}

\section{Experiments}
Finally, we present a variety of experiments to showcase the generality and effectiveness of our perturbation sets learned with a CVAE. Our experiments for each dataset can be broken into two distinct types: the generative modeling problem of learning and evaluating a perturbation set, and the robust optimization problem of learning an adversarially robust classifier to this perturbation. We note that our approach is broadly applicable, has no specific requirements for the encoder, decoder, and prior networks, and avoids the unstable training dynamics found in GANs. Furthermore, we do not have the blurriness typically associated with VAEs since we are modeling perturbations and not the underlying image. 
All code, configuration files, and pretrained model weights for reproducing our experiments are at \url{https://github.com/locuslab/perturbation_learning}. 

In all settings, we first train perturbation sets and evaluate them with a number of metrics averaged over the test set. We present a condensed version of these results in Table \ref{table:all_evaluate}, which establishes a quantitative baseline for learning real-world perturbation sets in three benchmark settings that future work can improve upon. Specifically, the approximation error measures the necessary subset property, the expected approximation error measures the sufficient likelihood property, and the reconstruction error and KL divergence are standard CVAE metrics. The full evaluation is described in Appendix \ref{app:perturbation}, and a complete tabulation of our results on evaluating perturbation sets can be found in Tables \ref{table:mnist_evaluate}, \ref{table:cifar10c_evaluate}, and \ref{table:mi_evaluate} in the appendix for each setting. 

We then leverage our learned perturbation sets into new downstream robustness tasks by simply using standard, well-vetted techniques for $\ell_2$ robustness directly on the latent space, namely adversarial training with an $\ell_2$ PGD adversary \citep{madry2017towards}, a certified defense with $\ell_2$ randomized smoothing \citep{cohen2019certified}, as well as an additional data-augmentation baseline via sampling from the CVAE truncated prior. Pseudo-code for these approaches can be found in Appendix \ref{app:pseudocode}. 
We defer the experiments on MNIST to Appendix \ref{app:mnist}, and spend the remainder of this section highlighting the main empirical results for the CIFAR10 and Multi-Illumination settings. Additional details and supplementary experiments can be found in Appendix \ref{app:cifar10c} for CIFAR10 common corruptions, and Appendix \ref{app:mi} for the multi-illumination dataset. 
%

\subsection{CIFAR10 common corruptions}
In this section, we first learn a perturbation set which captures common image corruptions for CIFAR10 \citep{hendrycks2019benchmarking}.\footnote{We note that this is not the original intended use of the dataset, which was proposed as a general measure for evaluating robustness. Instead, we are using the dataset for a different setting of learning perturbation sets.} We focus on the highest severity level of the blur, weather, and digital categories, resulting in 12 different corruptions which capture more ``natural'' corruptions that are unlike random $\ell_p$ noise (a complete description of the setting is in Appendix \ref{app:cifar10c}). 
We find that the resulting perturbation set accurately captures common corruptions with a mean approximation error of 0.005 as seen in Table \ref{table:all_evaluate}. A more in-depth quantitative evaluation as well as architecture and training details are in Appendix \ref{app:cifar10c_perturbation_set}. 


 
\begin{figure}[t]
  \centering
  \includegraphics[scale=0.95]{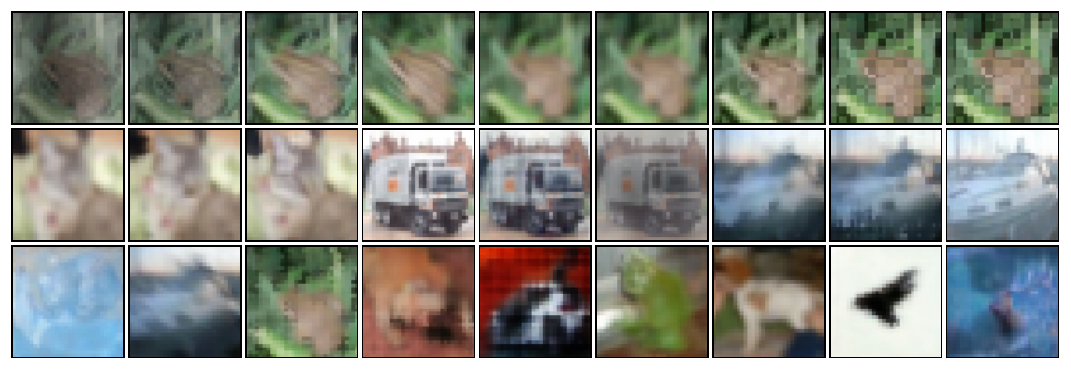}
  \caption{Visualization of a learned perturbation set trained on CIFAR10 common corruptions. (top row) Interpolations from fog, through defocus blur, to pixelate corruptions. (middle row) Random corruption samples for three examples. (bottom row) Adversarial corruptions that misclassify an adversarially trained classifier at $\epsilon=10.2$.}
  \label{fig:cifar10c_compact}
  \vspace{-0.1in}
\end{figure}
 
\begin{table}[t]
  \caption{Adversarial robustness to CIFAR10 common corruptions with a CVAE perturbation set.}
  \label{table:cifar10c_robustness}
  \centering
  \begin{tabular}{lrrrrrr}
    \toprule
    & \multicolumn{3}{c}{Test set accuracy $(\%)$} & \multicolumn{3}{c}{Test set robust accuracy $(\%)$}                   \\
    \cmidrule(r){2-4} \cmidrule(r){5-7}
    Method                    & Clean & Perturbed & OOD & $\epsilon=2.7$ & $\epsilon=3.9$ & $\epsilon=10.2$\\
    \midrule
    CIFAR10-C data augmentation             & $90.6$ & $87.7$ & $85.0$ & $42.4$ & $37.2$ & $17.8$ \\
    CVAE data augmentation        & $94.5$ & $\mathbf{90.5}$ & $89.6$ & $68.6$ & $63.3$ & $43.4$\\
    CVAE adversarial training     & $94.6$ & $90.3$ & $\mathbf{89.9}$ & $\mathbf{72.1}$ & $\mathbf{66.1}$ & $\mathbf{55.6}$\\
    \midrule
    Standard training           & $\mathbf{95.2}$ & $67.0$ & $68.1$ & $20.1$ & $17.8$ & $10.1$\\
    AugMix \citep{hendrycks2019augmix} & $92.0$ & $68.8$ & $82.9$ & $39.5$ & $34.4$ & $16.8$\\
    $\ell_2$ robust \citep{robustness} & $90.8$ & $74.4$ & $82.8$ & $58.4$ & $48.1$ & $20.6$\\
    $\ell_\infty$ robust \citep{carmon2019unlabeled} & $89.7$ & $71.2$ & $80.3$ & $60.2$ & $50.7$ & $23.6$ \\
    \bottomrule
  \end{tabular}
\end{table}

We qualitatively evaluate our perturbation set in Figure \ref{fig:cifar10c_compact}, which depicts interpolations, random samples, and adversarial examples from the perturbation set. For additional analysis of the perturbation set, we refer the reader to Appendix \ref{app:cifar10_pairing} for a study on using different pairing strategies during training and Appendix \ref{app:cifar10_latent} which finds semantic latent structure and visualizes additional examples. 

\paragraph{Robustness to corruptions}
We next employ the perturbation set in adversarial training and randomized smoothing to learn models which are robust against worst-case CIFAR10 common corruptions. We report results at three radius thresholds $\{2.7, 3.9, 10.2\}$ which correspond to the $25$th, $50$th, and $75$th percentiles of latent encodings as described in Appendix \ref{app:cifar10_latent}. We compare to two data augmentation baselines of training on perturbed data or samples drawn from the learned perturbation set, and also evaluate performance on three extra out-of-distribution corruptions (one for each weather, blur, and digital category denoted OOD) that are not present during training. 

We highlight some empirical results in Table \ref{table:cifar10c_robustness}, where we first find that training with the CVAE perturbation set can improve generalization. Specifically, using the CVAE perturbation set during training achieves $3-5\%$ improved accuracy over training directly on the common corruptions (data augmentation) across all non-adversarial metrics. 
These gains motivate learning perturbation sets beyond the setting of worst-case robustness as a way to improve standard generalization.  Additionally, the CVAE perturbation set improves worst-case performance, with the adversarially trained model being the most robust at $66\%$ robust accuracy for $\epsilon=3.9$ whereas pure data augmentation only achieves $17.8\%$ robust accuracy. 
Finally, we include a comparison to models trained with standard training, AugMix data augmentation, $\ell_\infty$ adversarial training, and $\ell_2$ adversarial training, none which can perform as well as our CVAE approach. We note that this is not too surprising, since these approaches have different goals and data assumptions. Nonetheless, we include these results for the curious reader, with additional details and discussion in Appendix \ref{app:baselines}. For certifiably robust models with randomized smoothing, we defer the results and discussion to Appendix \ref{app:cifar10c_certified}.

\begin{figure}[t]
  \centering
  \includegraphics[scale=1]{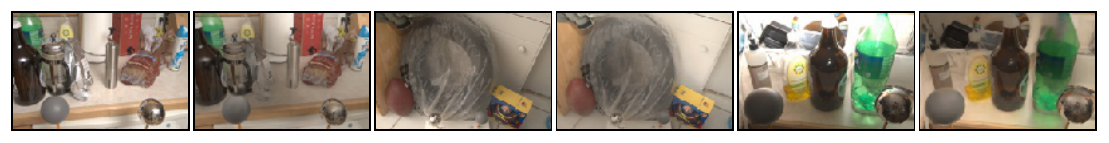}
  \caption{Pairs of MI scenes (left) and their adversarial lighting perturbations (right).}
  \label{fig:mi_adversarial_onerow}
  \vspace{-0.1in}
\end{figure}

\begin{table}[t]
  \caption{Learning image segmentation models that are robust to real-world changes in lighting with a CVAE perturbation set. }
  \label{table:mi_robustness}
  \centering
  \begin{tabular}{lrrrrrr}
    \toprule
    & \multicolumn{1}{c}{Test set accuracy (\%)} & \multicolumn{3}{c}{Test set robust accuracy (\%)}\\
    \cmidrule(r){2-2} \cmidrule(r){3-5}
    Method                    & Perturbed & $\epsilon=7.35$ & $\epsilon=8.81$ & $\epsilon=17$\\
    \midrule
    Fixed lighting angle      & $37.2$ & $26.3$ & $24.2$ & $14.9$\\
    MI data augmentation         & $45.2$ & $38.0$ & $36.5$ & $27.1$\\
    CVAE data augmentation    & $41.5$ & $35.5$ & $33.9$ & $24.7$\\
    CVAE adversarial training & $41.7$ & $39.4$ & $38.8$ & $35.4$\\
    \bottomrule
  \end{tabular}
\end{table}

\subsection{Multi-illumination}
Our last set of experiments looks at learning a perturbation set of multiple lighting conditions using the Multi-Illumination (MI) dataset \citep{murmann2019dataset}. These consist of a thousand scenes captured in the wild under 25 different lighting variations, and our goal is to learn a perturbation set which captures real-world lighting conditions. Since the process of learning lighting variations is largely similar to the image corruptions setting, we defer most of the discussion on learning and evaluating the CVAE-based lighting perturbation set to Appendix \ref{app:mi_perturbation_set}. Note that our perturbation set accurately captures real-world changes in lighting with a low approximation error of 0.006 as seen in Table \ref{table:all_evaluate}. 
We qualitatively evaluate our perturbation set by depicting adversarial examples in Figure \ref{fig:mi_adversarial_onerow}, with more visualizations (including samples and interpolations) in Appendix \ref{app:mi_visualizations}.

\paragraph{Robustness to lighting perturbations} We devote the remainder of this section to studying the task of generating material segmentation maps which are robust to lighting perturbations, using our CVAE perturbation set. We highlight that adversarial training improves robustness to worst-case lighting perturbations over directly training on the perturbed examples, increasing robust accuracy from $27.1\%$ to $35.4\%$ at the maximum radius $\epsilon=17$. Additional results on certifiably robust models with randomized smoothing can be found in Appendix \ref{app:mi_robustness}. 



\section{Conclusion}
In this paper, we presented a general framework for learning perturbation sets from data when the perturbation cannot be mathematically-defined. We outlined deterministic and probabilistic properties that measure how well a perturbation set fits perturbed data, and formally proved that a perturbation set based upon the CVAE framework satisfies these properties. This work establishes a principled baseline for learning perturbation sets with quantitative metrics which future work can potentially improve upon, e.g. by using a different generative modeling frameworks. The resulting perturbation sets open up new downstream robustness tasks such as adversarial and certifiable robustness to common image corruptions and lighting perturbations, while also potentially improving non-adversarial robust performance to natural perturbations. Our work opens a pathway for practitioners to learn machine learning models that are robust to targeted, real-world perturbations that can be collected as data. 

\bibliography{iclr2021_conference}
\bibliographystyle{iclr2021_conference}

\appendix
\section{Appendix}
\section{Theoretical results}
In this section, we present the theoretical results in their full detail and exposition. Both of the main theorems presented in this work require a number of preceding results in order to formally link the prior and the posterior distribution based on their KL divergence. We will present and prove these supporting results before proving each main theorem. 

\label{app:proofs}
\subsection{Proof of Theorem \ref{thm:cvae}}
Theorem \ref{thm:cvae} connects the CVAE objective to the necessary subset property. In order to prove this, we first prove three supporting lemmas. Lemma \ref{lem:normal} states that if the expected value of a function over a normal distribution is low, then for any fixed radius there must exist a point within the radius with proportionally low function value. This leverages the fact that the majority of the probability mass is concentrated around the mean and can be characterized by the Mahalanobis distance. 
\begin{lemma}
\label{lem:normal}
Let $f(u):\mathbb R^k \rightarrow \mathbb R_+$ be a non-negative integrable function, and let $\mathcal N_k(0,I)$ be a $k$-dimensional standard multivariate normal random variable with zero mean and identity covariance. Suppose $\mathbb E_{u \sim \mathcal N(0,I)}[f(u)] \leq \delta$ for some $\delta > 0$. Then, for any pair $(r,\alpha)$ where $r$ is the Mahalanobis distance which captures $1-\alpha$ of the probability mass of $\mathcal N_k(0,I)$, there exists a $u$ such that $\|u\|_2 \leq r$ and $f(u) \leq \frac{\delta}{1-\alpha}$. 
\end{lemma}
\begin{proof}
We will prove this by contradiction. Assume for sake of contradiction that for all $u$ such that $\|u\|_2\leq r$, we have $f(u)  > \frac{\delta}{1-\alpha}$. We divide the expectation into two integrals over the inside and outside of the Mahalanobis ball: 
\begin{equation}
\mathbb E_{u \sim \mathcal N(0,I)}\left[f(u) \right] = \int_{\|u\|_2\leq r}f(u) p(u)du + \int_{\|u\|_2 > r}f(u) p(u)du
\end{equation}
Using the assumption on the first integral and non-negativity of $f$ in the second integrand, we can conclude
\begin{equation}
\mathbb E_{u \sim \mathcal N(0,I)}\left[f(u)\right] > \int_{\|u\|_2\leq r} \frac{\delta}{1-\alpha}p(u) du = \delta
\end{equation}
where the equality holds by definition of the Mahalanobis distance, which contradicts the initial assumption that $\mathbb E_{u \sim \mathcal N(0,I)}\left[f(u)\right] \leq \delta$. Thus, we have proven by contradiction that there exists a $u$ such that $\|u\|_2\leq r$ and $f(u) \leq \frac{\delta}{1-\alpha}$.
\end{proof}

Our second lemma, Lemma \ref{lem:lambertw}, is an important result that comes from the algebraic form of the KL divergence. It is needed to connect the bound on the KL divergence to the actual variances of the prior and the posterior distribution, and uses the LambertW function to do so. 
\begin{lemma}
\label{lem:lambertw}
Let $W_k$ be the LambertW function with branch $k$. Let $x > 0$, and suppose $x - \log x \leq y$. Then, $x \in [a,b]$ where $a = -W_0(-e^{-y})$ and $b=-W_{-1}(-e^{-y})\}$. Additionally, these bounds coincide at $x=1$ when $y=1$. 
\end{lemma}
\begin{proof}
The LambertW function $W_k(y)$ is defined as the inverse function of $y = xe^x$, where the path to multiple solutions is determined by the branch $k$. We can then write the inverse of $y = x - \log x$ as one of the following two solutions: 
\begin{equation}
x \in \{-W_0(-e^{-y}), -W_{-1}(-e^{-y})\}
\end{equation} 
Since $x-\log x$ is convex with minimum at $x=1$, and since the two solutions surround $x=1$, the set of points which satisfy $x - \log x \leq y$ is precisely the interval $[-W_0(-e^{-y}), -W_{-1}(-e^{-y})]$. Evaluating the bound at $y=1$ completes the proof. 
\end{proof}

Lemma \ref{lem:kl} is the last lemma needed to prove Theorem \ref{thm:cvae}, which explicitly bounds terms involving the mean and variance of the prior and posterior distributions by their KL distance, leveraging Lemma \ref{lem:lambertw} to bound the ratio of the variances. The two quantities bounded in this lemma will be used in the main theorem to bound the distance of a point with low reconstruction error from the prior distribution. 
\begin{lemma}
\label{lem:kl}
Suppose the KL distance between two normals is bounded, so $KL(\mathcal N(\mu_1, \sigma^2_1)\|\mathcal N(\mu_2, \sigma^2_2)) \leq K$ for some constant $K$. Then, 
$$(\mu_1 - \mu_2)^2\frac{1}{\sigma_2^2} \leq K$$
and also $\frac{\sigma_1^2}{\sigma_2^2} \in [a,b]$ where
$$a = -W_0(-e^{-(K + 1)}), \quad b= -W_{-1}(-e^{-(K +1)})$$
\end{lemma}
\begin{proof}
By definition of KL divergence, we have
\begin{equation}
\frac{\sigma_1^2}{\sigma_2^2} + (\mu_1 - \mu_2)^2\frac{1}{\sigma_2^2} - 1 - \log \frac{\sigma_1^2}{\sigma_2^2} \leq K
\end{equation}
Since $x - \log x \geq 1$ for $x\geq 0$, we apply this for $x = \frac{\sigma_1^2}{\sigma_2^2}$ to prove the first bound on the squared distance as  
\begin{equation}
(\mu_1 - \mu_2)^2\frac{1}{\sigma_1^2} \leq K 
\end{equation}
Next, since $(\mu_1 - \mu_2)^2\frac{1}{\sigma_2^2} \geq 0$, we can bound the remainder as the following 
\begin{equation}
\frac{\sigma_1^2}{\sigma_2^2} - \log \frac{\sigma_1^2}{\sigma_2^2} \leq K + 1
\end{equation}
and using Lemma \ref{lem:lambertw}, we can bound $\frac{\sigma_1^2}{\sigma_2^2} \in [a, b]$ where 
$$a = -W_0(-e^{-(K + 1)}), \quad b= -W_{-1}(-e^{-(K +1)})$$
\end{proof}

With these three results, we can now prove the first main theorem, which we is presented below in its complete form, allowing us to formally tie the CVAE objective to the existence of a nearby point with low reconstruction error. 
\setcounter{theorem}{0}
\begin{theorem}
\label{thm:cvae}
Consider the likelihood lower bound for $x\in \mathbb R^n$ from the conditional VAE in $k$-dimensional latent space conditioned on some other input $y$. Let the posterior, prior, and decoder distributions be standard multivariate normals with diagonal covariance as follows 
$$q(z|x,y) \sim \mathcal N(\mu(x,y), \sigma^2(x,y)), \quad p(z|y) \sim \mathcal N(\mu(y), \sigma^2(y)), \quad p(x|z,y) \sim \mathcal N(g(z,y), I)$$
resulting in the following likelihood lower bound:
\begin{equation}
\log p(x|y) \geq \mathbb E_{q(z|x,y)}[\log p(x|z,y)] - KL(q(z|x,y) \| p(z|y))
\end{equation}
Suppose we have trained the lower bound to some thresholds $R,K_i$ 
$$\mathbb E_{q(z|x,y)}[\log p(x|z,y)] \geq R, \quad KL(q(z|x,y) \| p(z|y)) \leq \sum_{i=1}^k K_i$$ 
where $K_i$ bounds the KL-divergence of the $i$th dimension. Let $r$ be the Mahalanobis distance which captures $1-\alpha$ of the probability mass for a $k$-dimensional standard multivariate normal for some $0 < \alpha < 1$. Then, there exists a $z$ such that $\|\frac{z - \mu(y)}{\sigma(y)}\|_2 \leq \epsilon$ and $\|g(z,y) - x\|^2_2 \leq \delta$ for 
$$\epsilon = Br + \sqrt{\sum_i K_i}, \quad \delta = -\frac{1}{1-\alpha}\left(2R + m\log(2\pi)\right)$$ 
where $B$ is a constant dependent on $K_i$. Moreover, as $R\rightarrow -\frac{1}{2}m\log(2\pi)$ and $K_i \rightarrow 0$ (the theoretical limits of these bounds, e.g. by training), then $\epsilon \rightarrow r$ and $\delta \rightarrow 0$. 
\end{theorem}
\begin{proof}
The high level strategy for this proof will consist of two main steps. First, we will show that there exists a point near the encoding distribution which has low reconstruction error, where we leverage the Mahalanobis distance to capture nearby points with Lemma \ref{lem:normal}. Then, we apply the triangle inequality to bound its distance under the prior encoding. Finally, we will use Lemma \ref{lem:kl} to bound remaining quantities relating the distances between the prior and encoding distributions to complete the proof. 


Let $z(u) = u\cdot \sigma(x,y) + \mu(x,y)$ be the parameterization trick of the encoding distribution. Rearranging the log probability in the assumption and using the reparameterization trick, we get 
\begin{equation}
\mathbb E_{u}\left[\|x-g(z(u),y)\|_2^2\right] \leq - 2R - m \log(2\pi) = \delta(1-\alpha)
\end{equation}
Applying Lemma \ref{lem:normal}, there must exist a $u$ such that $\|u\|\leq r$ and $\|x - g(z(u),y)\|_2^2 \leq \delta$, and so $z(u)$ satisfies the reconstruction criteria for $\delta$. We will now show that $z(u)$ fulfills the remaining criteria for $\epsilon$: calculating its $\ell_2$ norm under the prior distribution and applying the triangle inequality, we get
\begin{equation}
\left\|\frac{z(u) - \mu(y)}{\sigma(y)}\right\|_2 = \left\|\frac{ u \cdot \sigma(x,y) + \mu(x,y) - \mu(y)}{\sigma(y)}\right\|_2 \leq \left\|\frac{ u \cdot \sigma(x,y)}{\sigma(y)}\right\|_2 + \left\|\frac{\mu(x,y) - \mu(y)}{\sigma(y)}\right\|_2
\end{equation}

We can use Lemma \ref{lem:kl} on the KL assumption to bound the following quantities
\begin{equation}
(\mu_i(y) - \mu_i(x,y))^2\frac{1}{\sigma_i^2(y)} \leq K_i, \quad \frac{\sigma_i^2(x,y)}{\sigma_i^2(y)} \in [a_i, b_i]
\end{equation} 
where $[a_i,b_i]$ are as defined from Lemma \ref{lem:kl}. 

Let $B = \max_i \sqrt{b_i}$, so $\frac{\sigma_i(x,y)}{\sigma_i(y)} \leq B$ for all $i$. Plugging this in along with the previous bounds we get the following bound on the norm of $z(u)$ before the prior reparameterization: 
\begin{equation}
\left\|\frac{z(u) - \mu(y)}{\sigma(y)}\right\|_2 \leq Br + \sqrt{\sum_i K_i} = \epsilon
\end{equation}
Thus, the norm (before the prior reparameterization) and reconstruction error of $z(u)$ can be bounded by $\epsilon$ and $\delta$. 

To conclude the proof, we note that from Lemma \ref{lem:lambertw}, $K_i \rightarrow 0$ for all $i$ implies $B \rightarrow 1$, and so $\epsilon \rightarrow r$. Similarly by inspection, $R\rightarrow -\frac{1}{2}m\log(2\pi)$ implies that $\delta \rightarrow 0$, which concludes the proof. 
\end{proof}

\subsection{Proof of Theorem \ref{thm:cvaesuperset}}
In this section, we move on to prove the second main result of this paper, Theorem \ref{thm:cvaesuperset}, which connects the CVAE objective to the sufficient likelihood property. The proof for this theorem is not immediate, because since the generator is an arbitrary function, two normal distributions which have a difference in means have an exponentially growing ratio in their tail distributions, and so truncating the normal distributions is crucial. This truncation is leveraged in Lemma \ref{lem:densityratio}, which bounds the ratio of two normal distributions constrained within an $\ell_2$ ball, and is what allows us to connect the expectation over the prior with the expectation over the posterior. 
\begin{lemma}
\label{lem:densityratio}
Let $p \sim \mathcal N(0,1)$ and $q \sim \mathcal N(-\mu/\sigma^2, 1/\sigma^2)$. Then, 
\begin{equation}
\frac{q(z)}{p(z)}1(|\sigma z + \mu| \leq r) \leq h(r, \mu, \sigma). 
\end{equation}
Furthermore if $\mu^2 \leq K$ and $\sigma \in [a,b]$, then $h(r, \mu, \sigma) \leq be^{\max(C_1, C_2)}$ where 
$$C_1 = (b^2 -1)r^2 - K, \quad C_2 = \frac{1}{a^2}\left((1-a^2)r^2 + 2r\sqrt{K} + K)\right).$$
\end{lemma}
\begin{proof}
The proof here is almost purely algebraic in nature. By definition of $q$ and $p$, we have 
\begin{equation}
\frac{q(z)}{p(z)} = \sigma e^{-\frac{1}{2}((\sigma z + \mu)^2  - z^2)} =  \sigma e^{\frac{1}{2}((1-\sigma ^2)z^2  - 2\sigma \mu z - \mu^2)}
\end{equation}
We will focus on bounding the exponent, $(1-\sigma ^2)z^2  - 2\sigma \mu z - \mu^2$. 
We can bound this by considering two cases. First, suppose that $1-\sigma^2 < 0$ and so the exponent is a concave quadratic. Then, the maximum value of the quadratic is at its root: 
\begin{equation}
2(1-\sigma^2)z - 2\sigma \mu= 0 \Rightarrow z = \frac{\sigma \mu}{1-\sigma^2}
\end{equation}
Further assume that this $z$ is within the interval of the indicator function, so $\left|\frac{\sigma^2\mu}{1-\sigma^2} + \mu\right| = \left|\frac{\mu}{1-\sigma^2} \right| \leq r$. Then, plugging in the maximum value into the quadratic results in the following bound $f_1$ for this case: 
\begin{equation}
(1-\sigma ^2)z^2  - 2\sigma \mu z - \mu^2 \leq -(1-\sigma^2)\frac{\sigma^2\mu^2}{(1-\sigma^2)^2} - \mu^2 \leq (\sigma^2-1)r^2 - \mu^2 \equiv f_1(r, \mu, \sigma)
\end{equation}
Consider the other case, so either $1-\sigma^2 \geq 0$, or the optimal value for $z^*$ in the previous case when $1-\sigma^2 < 0$ is not within the interval $|\sigma z + \mu| \leq r$. Then, the maximum value of this quadratic must occur at $|\sigma z + \mu| = r$, so $z = \frac{\pm r - \mu}{\sigma}$. Plugging this into the quadratic for positive $r$, this results in 
\begin{equation}
\left.(1-\sigma ^2)z^2  - 2\sigma \mu z - \mu^2\right|_{\sigma z + \mu = r} = \frac{1}{\sigma^2}\left((1-\sigma ^2)r^2 -  2r\mu  + \mu^2  \right)
\end{equation}
and for negative $r$, this is 
\begin{equation}
\left.(1-\sigma ^2)z^2  - 2\sigma \mu z - \mu^2\right|_{\sigma z + \mu = -r}  = \frac{1}{\sigma^2}\left((1-\sigma ^2)r^2  + 2r\mu  + \mu^2\right)
\end{equation}
and so we can bound this case with the following function $f_2$: 
\begin{equation}
(1-\sigma ^2)z^2  - 2\sigma \mu z - \mu^2 \leq \frac{1}{\sigma^2}\left((1-\sigma ^2)r^2  + 2|r\mu|  + \mu^2\right) \equiv f_2(r, \mu, \sigma)
\end{equation}
Plugging in the maximum over both cases forms our final bound on the ratio of distributions. 
\begin{equation}
\frac{q(z)}{p(z)}1(|\sigma z + \mu| \leq r) \leq \sigma e^{\max(f_1(r, \mu, \sigma), f_2(r, \mu, \sigma))} = h(r, \mu, \sigma)
\end{equation}

To finish the proof, assume we have the corresponding bounds on $\mu$ and $\sigma$. Then, the first case can be bounded with $C_1$ defined as 
$$f_1(r,\mu,\sigma) \leq (b^2 -1)r^2 - K = C_1$$
The second case can be bounded with $C_2$ defined as 
$$f_2(r,\mu,\sigma) \leq \frac{1}{a^2}\left((1-a^2)r^2 + 2r\sqrt{K} + K)\right) = C_2$$
And thus we can bound $h(r,\mu,\sigma)$ as 
$$h(r,\mu,\sigma) \leq be^{\max(C_1, C_2)} = H$$
\end{proof}

Lemma \ref{lem:densityratio} can be directly applied to the setting of Theorem \ref{thm:cvaesuperset} for the case of a non-conditional VAE with a standard normal distribution for the prior. However, since we are using conditional VAEs instead, the prior distribution has its own mean and variance and so the posterior distribution needs to be unparameterized from the posterior and reparameterized with the prior. Corollary \ref{cor:densityratio} establishes this formally, extending Lemma \ref{lem:densityratio} to the conditional setting. 
\begin{corollary}
\label{cor:densityratio}
Let $p \sim \mathcal N(0,1)$ and $q \sim \mathcal N\left(\frac{\mu_2-\mu_1}{\sigma_1}, \frac{\sigma_2^2}{\sigma_1^2}\right)$. Then, 
\begin{equation}
\frac{q(z)}{p(z)}1\left(\left|\frac{\sigma_1 z + \mu_1 - \mu_2}{\sigma_2}\right| \leq r\right) \leq h(r, \mu, \sigma)
\end{equation}
where $\mu = \frac{1}{\sigma_2}(\mu_1 - \mu_2)$ and $\sigma = \frac{\sigma_1}{\sigma_2}$. Furthermore, if $KL(\mathcal N(\mu_1, \sigma_1^2)\|\mathcal N(\mu_2, \sigma_2^2)) \leq K$ for some constant $K$, then $h(r,\mu,\sigma) \leq H$ where $H$ is a constant which depends on $K$. 
\end{corollary}
\begin{proof}
First, note that Lemma \ref{lem:kl} implies that $\mu^2 \leq K$ and $\sigma \in [a,b]$. Observe that $q\sim \mathcal N\left(-\frac{\mu}{\sigma}, \frac{1}{\sigma^2}\right)$ and $\frac{\sigma_1 z + \mu_1 - \mu_2}{\sigma_2} = \sigma z + \mu$, and so the proof reduces to an application of Lemma \ref{lem:densityratio} to this particular $\mu$ and $\sigma$. 
\end{proof}

%
These results allow us to bound the ratio of the prior and posterior distributions within a fixed radius, which will allow us to bound the expectation over the prior with the expectation over the posterior. We can now finally prove Theorem 2, presented in its complete form below. 
\begin{theorem}
\label{thm:cvaesuperset}
Consider the likelihood lower bound for $x\in \mathbb R^n$ from the conditional VAE in $k$-dimensional latent space conditioned on some other input $y$. Let the posterior, prior, and decoder distributions be standard multivariate normals with diagonal covariance as follows 
$$q(z|x,y) \sim \mathcal N(\mu(x,y), \sigma^2(x,y)), \quad p(z|y) \sim \mathcal N(\mu(y), \sigma^2(y)), \quad p(x|z,y) \sim \mathcal N(g(z,y), I)$$
resulting in the following likelihood lower bound:
\begin{equation}
\log p(x|y) \geq \mathbb E_{q(z|x,y)}[\log p(x|z,y)] - KL(q(z|x,y) \| p(z|y))
\end{equation}
Suppose we have trained the lower bound to some thresholds $R,K_i$ 
$$\mathbb E_{q(z|x,y)}[\log p(x|z,y)] \geq R, \quad KL(q(z|x,y) \| p(z|y)) \leq \sum_{i=1}^k K_i$$ 
where $K_i$ bounds the KL-divergence of the $i$th dimension. Let $r$ be the Mahalanobis distance which captures $1-\alpha$ of the probability mass for a $k$-dimensional standard multivariate normal for some $0 < \alpha < 1$. Then, the \emph{truncated expected reconstruction error} can be bounded with 
$$\mathbb E_{p_r(u)}\left[\|g(u\cdot\sigma(y) + \mu(y),y) - x\|_2^2\right] \leq - \frac{1}{1-\alpha}(2R + m \log(2\pi))H$$
where $p_r(u)$ is a multivariate normal that has been truncated to radius $r$ and $H$ is a constant that depends exponentially on $K_i$.

\end{theorem}
\begin{proof}
The overall strategy for this proof will be as follows. First, we will rewrite the expectation over a truncated normal into an expectation over a standard normal, using an indicator function to control the radius. Second, we will do an ``unparameterization'' of the standard normal to match the parameterized objective in the assumption. Finally, we will bound the ratio of the unparameterized density over the normal prior, which allows us to bound the expectation over the prior with the assumption. This last step to bound the ratio of densities is made possible by the truncation, which would otherwise grow exponentially with the tails of the distribution. 

For notational simplicity, let $f(\cdot) = \|g(\cdot,y) - x\|_2^2$. The quantity we wish to bound can be rewritten using the Mahalanobis distance as  
\begin{equation}
\mathbb E_{p_\alpha} \left[ f(u\cdot\sigma(y) + \mu(y)) \right]= \frac{1}{1-\alpha}\int_{u} f(u\cdot\sigma(y) + \mu(y))1(\|u\| \leq r)p(u) dz
\end{equation}
where we used the fact that $p_r(u) = \frac{1}{1-\alpha}1(\|u\|\leq r)p(u)$, which simply rewrites the density of a truncated normal using a scaled standard normal density and an indicator function. We can do a parameterization trick $z = u \cdot \sigma(y) + \mu(y)$ to rewrite this as 
\begin{equation}
\mathbb E_{p_r} \left[ f(u\cdot\sigma(y) + \mu(y)) \right]= \frac{1}{1-\alpha}\int_{z} f(z)1\left(\left\|\frac{z - \mu(y)}{\sigma(y)}\right\| \leq r\right)p_{z|y}(z) dz
\end{equation}
followed by a reverse parameterization trick with $v = \frac{z - \mu(x,y)}{\sigma(x,y)}$ to get the following equivalent expression
\begin{equation}
\frac{1}{1-\alpha}\int_{z} f(v\cdot \sigma(x,y) + \mu(x,y))1\left(\left\|\frac{v\cdot \sigma(x,y) + \mu(x,y) - \mu(y)}{\sigma(y)}\right\| \leq r\right)p_{v}(v) dz
\end{equation}
where $p_v \sim \mathcal N\left(\frac{\mu(x)-\mu(x,y)}{\sigma(x,y)}, \frac{\sigma^2(x)}{\sigma^2(x,y)}\right)$. For convenience, we can let 
$$\hat{p}_v(v) = 1\left(\left\|\frac{v\cdot \sigma(x,y) + \mu(x,y) - \mu(y)}{\sigma(y)}\right\| \leq r\right)p_{v}(v)$$
which can be interpreted as a truncated version of $p_v$, and so the expectation can be represented more succinctly as 
\begin{equation}
\label{eq:expectationbeforebound}
\mathbb E_{p_r} \left[ f(u\cdot\sigma(y) + \mu(y)) \right]= \frac{1}{1-\alpha}\int_{z} f(v\cdot \sigma(x,y) + \mu(x,y))\hat p_{v}(v) dz.
\end{equation}
We will now bound $\hat p_v(v)$ with the standard normal distribution $p(v)$ so that we can apply our assumption. Because $f$ is non-negative, the $\ell_2$ constraint in $\hat p_v(v)$ can be relaxed to an element-wise $\ell_\infty$ constraint to get 
\begin{equation}
\hat p_v(v)\leq p(v)\prod_{i=1}^k 1\left(\left|\frac{v\cdot \sigma(x,y) + \mu(x,y) - \mu(y)}{\sigma(y)}\right| \leq r\right)\frac{p_v(v_i)}{p(v_i)}
\end{equation}
where we also used the fact that $p$ is a diagonal normal, so $\prod_{i=1}^k p(v_i) = p(v)$. Each term in this product can be bounded by Lemma \ref{lem:densityratio} to get 
\begin{equation}
\label{eq:boundpdf}
\hat p_v(v) \leq p(v)\prod_{i=1}^k H_i = p(v)\cdot H
\end{equation}
where $H_i$ is as defined in Corollary \ref{cor:densityratio} for each $i$ using the corresponding $K_i$, and we let $H = \prod_{i=1}^k H_i$. 
Thus we can now bound the truncated expected value by plugging in our bound for $\hat p_v(v)$ into Equation \eqref{eq:expectationbeforebound} to get
\begin{equation}
\label{eq:truncated_expectation_bound}
\mathbb E_{p_r} \left[ f(u\cdot\sigma(y) + \mu(y)) \right]\leq  \frac{H}{1-\alpha}\cdot \mathbb E_{q(z|x,y)} \left[ f(z)\right]
\end{equation}
Using our lower bound on the expected log likelihood from the assumption, we can bound the remaining integral with 
\begin{equation}
\label{eq:boundrecon}
\mathbb E_{q(z|x,y)}\left[f(z)\right] \leq -2R - m\log(2\pi)
\end{equation}
and so combining this into our previous bound from Equation \eqref{eq:truncated_expectation_bound} results in the final bound
\begin{equation}
E_{p_r} \left[ f(u) \right] \leq  - \frac{1}{1-\alpha}(2R + m \log(2\pi))H
\end{equation}
\end{proof}

\begin{figure}[t]
  \centering
\ifx\du\undefined
  \newlength{\du}
\fi
\setlength{\du}{8\unitlength}
\begin{tikzpicture}
\pgftransformxscale{1.000000}
\pgftransformyscale{-1.000000}
\definecolor{dialinecolor}{rgb}{0.000000, 0.000000, 0.000000}
\pgfsetstrokecolor{dialinecolor}
\definecolor{dialinecolor}{rgb}{1.000000, 1.000000, 1.000000}
\pgfsetfillcolor{dialinecolor}
\pgfsetlinewidth{0.080000\du}
\pgfsetdash{}{0pt}
\pgfsetdash{}{0pt}
\pgfsetbuttcap
{
\definecolor{dialinecolor}{rgb}{0.000000, 0.000000, 0.000000}
\pgfsetfillcolor{dialinecolor}
\definecolor{dialinecolor}{rgb}{0.000000, 0.000000, 0.000000}
\pgfsetstrokecolor{dialinecolor}
\draw (35.500000\du,8.000000\du)--(34.500000\du,8.000000\du);
}
\pgfsetlinewidth{0.080000\du}
\pgfsetdash{}{0pt}
\pgfsetdash{}{0pt}
\pgfsetbuttcap
{
\definecolor{dialinecolor}{rgb}{0.000000, 0.000000, 0.000000}
\pgfsetfillcolor{dialinecolor}
\definecolor{dialinecolor}{rgb}{0.000000, 0.000000, 0.000000}
\pgfsetstrokecolor{dialinecolor}
\draw (38.000000\du,19.500000\du)--(38.000000\du,20.500000\du);
}
\pgfsetlinewidth{0.080000\du}
\pgfsetdash{}{0pt}
\pgfsetdash{}{0pt}
\pgfsetbuttcap
{
\definecolor{dialinecolor}{rgb}{0.000000, 0.000000, 0.000000}
\pgfsetfillcolor{dialinecolor}
\definecolor{dialinecolor}{rgb}{0.000000, 0.000000, 0.000000}
\pgfsetstrokecolor{dialinecolor}
\draw (32.000000\du,19.500000\du)--(32.000000\du,20.500000\du);
}
\pgfsetlinewidth{0.100000\du}
\pgfsetdash{}{0pt}
\pgfsetdash{}{0pt}
\pgfsetbuttcap
{
\definecolor{dialinecolor}{rgb}{0.000000, 0.000000, 0.000000}
\pgfsetfillcolor{dialinecolor}
\pgfsetarrowsstart{stealth}
\pgfsetarrowsend{stealth}
\definecolor{dialinecolor}{rgb}{0.000000, 0.000000, 0.000000}
\pgfsetstrokecolor{dialinecolor}
\draw (20.000000\du,20.000000\du)--(50.000000\du,20.000000\du);
}
\pgfsetlinewidth{0.100000\du}
\pgfsetdash{}{0pt}
\pgfsetdash{}{0pt}
\pgfsetbuttcap
{
\definecolor{dialinecolor}{rgb}{0.000000, 0.000000, 0.000000}
\pgfsetfillcolor{dialinecolor}
\pgfsetarrowsstart{stealth}
\pgfsetarrowsend{stealth}
\definecolor{dialinecolor}{rgb}{0.000000, 0.000000, 0.000000}
\pgfsetstrokecolor{dialinecolor}
\draw (35.000000\du,22.000000\du)--(35.000000\du,5.000000\du);
}
\pgfsetlinewidth{0.200000\du}
\pgfsetdash{}{0pt}
\pgfsetdash{}{0pt}
\pgfsetmiterjoin
\pgfsetbuttcap
{
\definecolor{dialinecolor}{rgb}{0.000000, 0.388235, 0.505882}
\pgfsetfillcolor{dialinecolor}
\pgfsetarrowsstart{stealth}
\pgfsetarrowsend{stealth}
{\pgfsetcornersarced{\pgfpoint{0.000000\du}{0.000000\du}}\definecolor{dialinecolor}{rgb}{0.000000, 0.388235, 0.505882}
\pgfsetstrokecolor{dialinecolor}
\draw (22.000000\du,20.000000\du)--(32.000000\du,20.000000\du)--(35.000000\du,8.000000\du)--(38.000000\du,20.000000\du)--(48.000000\du,20.000000\du);
}}
\definecolor{dialinecolor}{rgb}{0.000000, 0.000000, 0.000000}
\pgfsetstrokecolor{dialinecolor}
\node[anchor=west] at (35.50000\du,8.044094\du){$\alpha$};
\definecolor{dialinecolor}{rgb}{0.000000, 0.000000, 0.000000}
\pgfsetstrokecolor{dialinecolor}
\node at (32.200000\du,21.159313\du){$-\alpha^{-2}$};
\definecolor{dialinecolor}{rgb}{0.000000, 0.000000, 0.000000}
\pgfsetstrokecolor{dialinecolor}
\node at (38.600000\du,21.159313\du){$\alpha^{-2}$};
\node at (50.800000\du,19.959313\du){$z$};
\node at (35.200000\du,3.959313\du){$\delta_\alpha(z)$};
\end{tikzpicture}
  \caption{A simple example demonstrating how the expected value of a function can tend to zero while the maximum tends to infinity as $a\rightarrow\infty$.}
  \label{fig:lineardelta}
\end{figure}
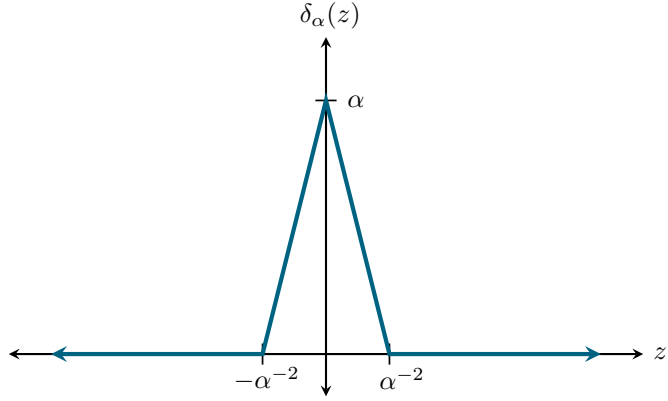

\subsection{Discussion of theoretical results}
\label{app:proof_discussion}
We first note that the bound on the expected reconstruction error in Theorem \ref{thm:cvaesuperset_short} is larger than the corresponding bound from Theorem \ref{thm:cvae_short} by an exponential factor. This gap is to some extent desired, since Theorem \ref{thm:cvae_short} characterizes the existence of a highly accurate approximation whereas Theorem \ref{thm:cvaesuperset_short} characterizes the average case, and so if this gap were too small, the average case of allowable perturbations would be constrained in reconstruction error, which is not necessarily desired in all settings. 

This relates to the overapproximation error, or the maximum reconstruction error within the perturbation set. Initially one may think that having low overapproximation error to be a desirable property of perturbation sets, in order to constrain the perturbation set from deviating by ``too much.'' However, perturbation sets such as the rotation-translation-skew perturbations for MNIST will always have high overapproximation error, and so reducing this is not necessary desired. Furthermore, it not theoretically guaranteed for the CVAE to minimize the overapproximation error without further assumptions. This is because it is possible for a function to be arbitrarily small in expectation but be arbitrarily high at some nearby point, so optimizing the CVAE objective does not imply low overapproximation error. A simple concrete example demonstrating this is the following piecewise linear function
$$\delta_a(z) = \begin{cases}
a^3z + a &\text{for } -\frac{1}{a^2} \leq z  < 0\\
-a^3z + a &\text{for } 0 \leq z < \frac{1}{a^2}\\
0 &\text{otherwise} 
\end{cases}$$
for $a>0$ as seen in Figure \ref{fig:lineardelta}. Note that $\max_{|z|\leq \epsilon} \delta_a(z) = a$ which is unbounded as $a\rightarrow \infty$ for all $\epsilon\geq 0$, while at the same time $\mathbb E_{\mathcal N(0,1)}[\delta_a(z)]\leq \frac{1}{a} \rightarrow 0$. Thus, explicitly minimizing overapproximation error is not necessarily desired nor is it guaranteed by the CVAE. Nonetheless, we report it anyways to give the reader a sense of how much the CVAE can deviate within the learned perturbation set. 



\section{Evaluating a perturbation set}
\label{app:perturbation}
In this section we describe in more detail the process of selecting a radius for a learned perturbation set and formally describe all the details regarding the evaluation metrics for a perturbation set learned with a CVAE. As a reminder, let $\mu(x,\tilde x), \sigma(x, \tilde x)$ be the posterior encoding, $\mu(x),\sigma(x)$ be the prior encoding, and $g(x,z)$ be the decoder for the CVAE. We will use $u$ to represent the latent space before the parameterization trick. Then the perturbation set defined by our generative process is 
$$\mathcal S(x) = \{ g(z,x) : z = u \cdot \sigma(x) + \mu(x),\;\|u\| \leq \epsilon\}$$
and the probabilistic perturbation set is defined by the truncated normal distribution $\mathcal N_\epsilon$ before the parameterization trick as follows, 
$$\mathcal S(x) \sim \textrm{Normal}(g(z,x), I),\quad z = u \cdot \sigma(x) + \mu(x),\quad u \sim \mathcal N_\epsilon(0,I)$$

\paragraph{Selecting a radius} To select a radius, we take the following conservative estimate calculated on a held-out validation set, which computes the smallest $\ell_2$ ball under the prior which contains all the mean encodings under the posterior: 
$$\epsilon = \max_{i} \left\| \frac{\mu(x_i,\tilde x_i) - \mu(x_i)}{\sigma(x_i)}\right\|_2$$
The benefits of such a selection is that by taking the maximum, we are selecting a learned perturbation set that includes every point with approximation error as low as the posterior encoder. To some extent, this will also capture additional types of perturbations beyond the perturbed data, which can be both beneficial and unwanted depending on what is captured. However, in the context of adversarial attacks, using a perturbation set which is too large is generally more desirable than one which is too small, in order to not underspecify the threat model. 

\paragraph{Evaluation metrics}
We present each evaluation metric in detail for a single perturbed pair $(x, \tilde x)$. These evaluation metrics can then be averaged over the test dataset to produce the evaluations presented in this paper. 
\begin{enumerate}
\item {\bf Encoder approximation error (Enc. AE)} We can get a fast upper bound of approximation error by taking the posterior mean, unparameterizing it with respect to the prior, and projecting it to the $\epsilon$ ball as follows: 
\begin{align*}
\textrm{Enc. AE}(x, \tilde x) &\equiv \|\tilde x - g(z,x)\|^2\\
\textrm{where}\;&  z =u \cdot \sigma(x) + \mu(x),\\
& u = \textrm{Proj}_\epsilon\left(\frac{\mu(x, \tilde x) - \mu(x)}{\sigma(x)}\right)
\end{align*}

\item {\bf PGD approximation error (PGD AE)} We can refine the upper bound by solving the following problem with projected gradient descent: 
\begin{align*}
\textrm{PGD AE}(x, \tilde x) &\equiv \min_{\|u\|\leq \epsilon} \|\tilde x - g(z,x)\|^2\\
\textrm{where}\;&  z =u \cdot \sigma(x) + \mu(x)
\end{align*}
In practice, we implement this by warm starting the procedure with the solution from the encoder approximation error, and run $50$ iterations of PGD at step size $\epsilon/20$. 

\item {\bf Expected approximation error (EAE)} We can compute this by drawing $N$ samples $u_i \sim \mathcal N_\epsilon(0,I,\epsilon)$ and calculating the following Monte Carlo estimate: 
\begin{align*}
\textrm{EAE}(x, \tilde x) &\equiv \frac{1}{N}\sum_{i=1}^N \|\tilde x - g(z_i,x)\|^2\\
\textrm{where}\;&  z_i =u_i \cdot \sigma(x) + \mu(x)
\end{align*}
In practice, we find that $N=5$ is sufficient for reporting means over the dataset with near-zero standard deviation. 

\item {\bf Over approximation error (OAE)} We can compute this by performing a typical $\ell_2$ PGD adversarial attack: 
\begin{align*}
\textrm{OAE}(x, \tilde x) &\equiv \max_{\|u\|\leq \epsilon} \|\tilde x - g(z,x)\|^2\\
\textrm{where}\;&  z =u \cdot \sigma(x) + \mu(x)
\end{align*}
In practice, we implement this by doing a random initialization and run $50$ iterations of PGD at step size $\epsilon/20$. 

\item {\bf Reconstruction error (Recon. err)} This is the typical reconstruction error of a variational autoencoder, which is a Monte Carlo estimate over the full posterior with one sample: 
\begin{align*}
\textrm{Recon. err}(x, \tilde x) &\equiv \frac{1}{m}\|\tilde x - g(z,x)\|^2\\
\textrm{where}\;&  z =u \cdot \sigma(x,\tilde x) + \mu(x, \tilde x)\\
& u \sim \mathcal N(0,I)
\end{align*}
Note that we report the average over all pixels to be consistent with the other metrics in this paper, however it is typical to implement this during training as a sum of squared error instead of a mean. 

\item {\bf KL divergence (KL)} This is the standard KL divergence between the posterior and the prior distributions
$$KL(x, \tilde x) \equiv KL(\mathcal{N}(\mu(x,\tilde x), \sigma(x,\tilde x)\|\mathcal{N}(\mu(x), \sigma(x)))$$
\end{enumerate}

\section{Adversarial training, randomized smoothing, and data augmentation with learned perturbation sets}
\label{app:pseudocode}
In this section, we describe how various standard, successful techniques in robust training can applied to learned perturbation sets. We note that each method is virtually unchanged, with the only difference being the application of the method to the latent space of the generator instead of directly in the input space. 

\begin{algorithm}
\caption{Given a dataset $D$, perform an epoch of adversarial training with a learned perturbation set given by a generator $g$ and a radius $\epsilon$ with step size $\gamma$, using a PGD adversary with $T$ steps and step size $\alpha$}
\label{alg:adv_training}
\begin{algorithmic}
  \For{$(x,y) \in D$}
  \State $\delta \coloneqq 0$ \emph{// Initialize perturbation}
  \For{$t \in T$}
  \State $g \coloneqq \nabla_\delta \ell(h(g(z+\delta,x)),y)$ \emph{// Calculate gradient}
  \State $\delta \coloneqq \delta + \alpha \cdot g / \|g\|_2$ \emph{// Take a gradient step}
  \If {$\|\delta\|_2 > \epsilon$}
  \State $\delta \coloneqq \epsilon \cdot \delta / \|\delta\|_2$ \emph{// Project onto $\epsilon$ ball}
  \EndIf
  \EndFor
  \State $\theta \coloneqq \theta - \gamma \nabla_\theta \ell(h(g(z+\delta,x)),y)$ \emph{// Optimize model weights}
  \EndFor
\end{algorithmic}
\end{algorithm}

\subsection{Adversarial training}
Adversarial training, or training on adversarial examples generated by an adversary, is a leading empirical defenses for learning models which are robust to adversarial attacks. The adversarial attack is typically generated with a PGD adversary and is used in the context of $\ell_p$ adversarial examples. We follow closely the adversarial training formulation from \citep{madry2017towards}. In order to do adversarial training on learned perturbation sets, it suffices to simply run the PGD adversarial attack in the latent space of the generator, as shown in Algorithm \ref{alg:adv_training}. 
%
%
\begin{algorithm}
\caption{Given a datapoint $x$, pseudocode for certification and prediction for a classifier which has been smoothed over a learned perturbation set given by a generator $g$ and a radius $\epsilon$ using a noise level $\sigma$ with probability at least $1-\alpha$. }
\label{alg:smoothing}
\begin{algorithmic}
  \Function{SampleUnderNoise}{$h,g,x,n,\sigma$}
  \State $c_i \coloneqq 0$ for $i \in [k]$ \emph{// Initialize counts}
  \For{$t \in [n]$}
  \State \emph{// Draw random samples from the generator and count classes for prediction}
  \State $z \coloneqq \mathcal N(0, \sigma^2)$ 
  \State $j \coloneqq \argmax_i h(g(z,x))$ 
  \State $c_j \coloneqq c_j + 1$ 
  \EndFor
  \State \Return $c$
  \EndFunction
  \State
  \Function{Predict}{$h,g,x,n,\sigma,\alpha$}
  \State $c \coloneqq$ \Call{SampleUnderNoise}{$h,g,x,n,\sigma$}
  \State $i,j \coloneqq$ top two indices in $c$
  \If {\Call{BinomPValue}{$c_i, c_i + c_j, 0.5$} $\leq \alpha$}
  \State \Return prediction $i$
  \Else 
  \State \Return ABSTAIN
  \EndIf
  \EndFunction
  \State
  \Function{Certify}{$h,g,x,n_0,n,\sigma,\alpha$}
  \State $c_0 \coloneqq$ \Call{SampleUnderNoise}{$h,g,x,n_0,\sigma$}
  \State $k \coloneqq \argmax_i (c_0)_i$
  \State $c \coloneqq$ \Call{SampleUnderNoise}{$h,g,x,n,\sigma$}
  \State $p_a \coloneqq$ \Call{LowerConfidenceBound}{$c_k',n,1-\alpha$}
  \If {$p_a > 0.5$}
  \State \Return prediction $k$ with radius $\sigma \Psi^{-1}(p_a)$
  \Else
  \State \Return ABSTAIN
  \EndIf
  \EndFunction
\end{algorithmic}
\end{algorithm}

\subsection{Randomized smoothing}
Randomized smoothing is a method for learning models robust to adversarial examples which come with certificates that can prove (with high probability) the non-existance of adversarial examples. We follow closely the randomized smoothing procedure from \citep{cohen2019certified}, which trains and certifies a network by augmented the inputs with large Gaussian noise. In order to do randomized smoothing on the learned perturbation sets, it suffices to simply train, predict, and certify with augmented noise in the latent space of the generator, as shown in Algorithm \ref{alg:smoothing} for prediction and certification. 

The algorithms are almost identical to that propose by \citet{cohen2019certified}, with the exception that the noise is passed to the generator before going through the classifier. Note that \textproc{LowerConfidenceBound}$(k,n,1-\alpha)$ returns a one-sided $1-\alpha$ confidence interval for the Binomial parameter $p$ given a sample $k\sim \textproc{Binomial}(n,p)$, and  \textproc{BinomPValue}$(a,a+b,p)$ returns the $p$ value of a two-sided hypothesis test that $a \sim \textproc{Binomial}(a+b,p)$. 

\section{MNIST}
\label{app:mnist}
In this section, we discuss some benchmark tasks on MNIST, where the perturbation set is mathematically-defined. Since an exact perturbation set will always be better than a learned approximation, the goal here is simply to explore the data and architectural requirements for learning accurate perturbations sets with CVAEs, and establish a baseline where the ground truth is known. We consider learning the classic $\ell_\infty$ perturbation 
and rotation-translation-skew (RTS) perturbations 
\citep{jaderberg2015spatial}. 

We use the standard MNIST dataset consisting of 60,000 training examples with 1,000 examples randomly set aside for validation purposes, and 10,000 examples in the test set. These experiments were run on a single GeForce RTX 2080 Ti graphics card, with the longest perturbation set taking 1 hour to train. 

\subsection{Experimental details}
\label{app:mnist_details}
\paragraph{$\ell_\infty$ details}
For the $\ell_\infty$ setting, perturbations with radius $\epsilon = 0.3$. Our fully connected network encoders have one hidden layer with $784$ hidden units, then two linear modules with $784$ outputs each for computing the mean and log variance of the latent space distributions. The generator network has the same structure, one hidden layer with $784$ hidden units. The posterior and generator networks simply concatenate their inputs into a single vector. 

The encoders for the convolutional network for the $\ell_\infty$ setting uses two hidden convolutional layers with $32$ and $64$ channels, each followed by ReLU and max pooling. To generate the mean and log variance, there are two final convolution with $128$ channels. The posterior network concatenates its inputs via the channel dimension. The generator has two analogous hidden layers using convolutions with 64 and 32 channels followed by ReLU and nearest neighbor upsampling. One final convolution reduces the number of channels to 1, and the conditioned input $x$ is downsampled with $1\times 1$ convolutions and concatenated with the feature map before each convolution. All convolutions other than the downsampling convolutions use a $3 \times 3$ kernel with one padding. 

Both networks are trained for 20 epochs, with step size following a piece-wise linear schedule of $[0, 0.001, 0.0005, 0.0001]$ over epochs $[0, 10, 15, 20]$ with the Adam optimizer using batch size 128. The KL divergence is weighted by $\beta$ which follows a piece-wise linear schedule from $[0,0.001,0.01]$ over epochs $[0,5,20]$.  

\paragraph{RTS details}
For the RTS setting, we follow the formulation studied by \citet{jaderberg2015spatial}, which consists of a random rotation between $[-45, 45]$ degrees, random scaling factor between $[0.7, 1.3]$, and random placement in a $42\times 42$ image. The architecture uses a typical pre-activation convolution block, which consists of 2 convolutions with reflection padding which are pre-activated with batchnorm and ReLU. The encoders uses a single block followed by two linear modules to $128$ units to generate the mean and log variance, where the posterior network concatenates its inputs via the channel dimension. The generator has a linear layer to map the latent vector back to a $42\times 42$ feature map, followed by five parallel spatial transformers which use the latent vector to transform the conditioned input. Multiple spatial transformers here are used simply because training just one can be inconsistent. The outputs of the spatial transformers are concatenated via the channel dimension, followed by a convolutional block and one final convolution to reduce it back to one channel. 

All networks are trained for 100 epochs with cyclic learning rate peaking at $0.0008$ on the $40$th epoch with the Adam optimizer using batch size 128. The KL divergence is weighted by $\beta$ which follows a piece-wise linear schedule from $[0,0.01,1,1]$ over epochs $[0,10,50,100]$.  

\begin{figure}[t]
  \centering
  \includegraphics[scale=1]{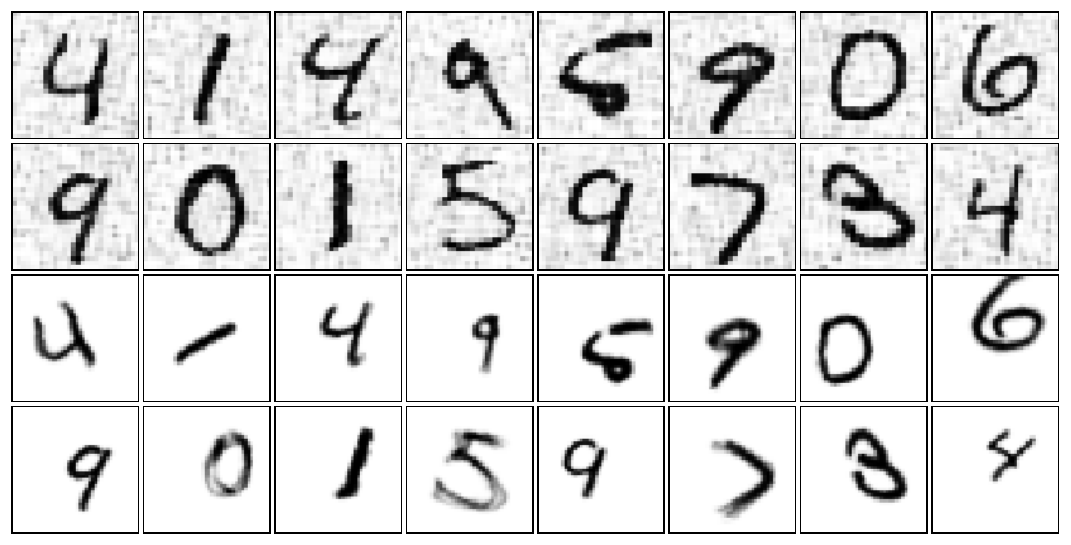}
  \caption{Samples of $\ell_\infty$ (top) and RTS (bottom) perturbations for MNIST from the convolutional $\ell_\infty$ and RTS-1 models.}
  \label{fig:mnist_samples}
\end{figure}

\subsection{Samples and visualizations}
\label{app:mnist_samples}
Figure \ref{fig:mnist_samples} plots samples from the convolutional model producing $\ell_\infty$ perturbations and the RTS model trained with only one perturbed datapoint per example. We note that the final $\beta$ value for the $\ell_\infty$ perturbation was tuned to produce reasonable samples in order to balance the KL divergence and reconstruction error, likely due to the inherent difficulty of this setting.

\begin{table}[t]
  \caption{Evaluation of a perturbation set learned from MNIST examples perturbed by $\ell_\infty$ noise with $\epsilon=0.3$ and RTS perturbations. The standard deviation is less than 0.01 for all metrics except KL divergence.}
  \label{table:mnist_evaluate}
  \centering
  \begin{tabular}{lcrrrrrr}
    \toprule
    && \multicolumn{4}{c}{Test set quality metrics} & \multicolumn{2}{c}{Test set CVAE metrics}                  \\
    \cmidrule(r){3-6}\cmidrule(r){7-8}
    Model                    & $\epsilon$ & Enc. AE & PGD AE & EAE & OAE & Recon. err & KL \\
    \midrule
    Fully connected $\ell_\infty$ & $28$ & $0.31$ & $0.25$ & $0.32$ & $0.65$ &  $0.27$ & $585.5 \pm 0.29$ \\
    Convolutional  $\ell_\infty$ & $29$ & $0.30$ & $0.27$ & $0.32$ & $0.35$ & $0.27$ & $610.3 \pm 0.10$ \\
    RTS         & $14$ & $0.29$ & $0.11$ & $0.54$ & $1.03$ & $0.04$ & $22.2\pm 0.02$ \\
    RTS-1       & $14$ & $0.28$ & $0.10$ & $0.61$ & $1.34$ & $0.05$ & $22.2 \pm 0.03$ \\
    RTS-5       & $14$ & $0.28$ & $0.10$ & $0.53$ & $0.86$ & $0.05$ & $23.7 \pm 0.02$ \\
    \bottomrule
  \end{tabular}
\end{table}

\subsection{Evaluating the perturbation set}
\label{app:mnist_evaluation}
Table \ref{table:mnist_evaluate} contains the full tabulated results evaluating the MNIST perturbation sets. The KL divergence for the $\ell_\infty$ perturbation sets are much higher due to the KL divergence being weighted by $0.01$ in order to produce reasonable samples. We note that the convolutional network and fully connected network evaluate to about the same, except that the convolutional network has a much lower overapproximation error and is likely restraining the size of the learned perturbation set. 

The RTS-1 results demonstrate how fixing the number of perturbations seen during training to one per datapoint is still enough to learn a perturbation set with as much approximation error as a perturbation set with an infinite number of samples, denoted RTS. Increasing the number of perturbations to five per datapoint allows the remaining metrics to match the RTS model, and so this suggests that not many samples are needed to learn a perturbation set in this setting. 


\section{CIFAR10 common corruptions}
\label{app:cifar10c}
In this section, we describe the CIFAR10 common corruptions setting and experiments in greater detail. The dataset comes with 15 common corruptions covering noise, blurs, weather, and digital corruptions. We omit the three noise corruptions due to similarity to unstructured $\ell_p$ noise, leaving us with the following 12 corruptions for each example: 
\begin{enumerate}
\item Blurs: defocus blur, glass blur, motion blur, zoom blur
\item Weather: snow, frost, fog
\item Digital: brightness, contrast, elastic, pixelate, jpeg
\end{enumerate}
The dataset comes with different corruption levels, so we focus on the highest severity, corruption level five, so the total training set has 600,000 perturbations (12 for each of 50,000 examples) and the test set has 120,000 perturbations. The dataset also comes with several additional corruptions meant to be used as a validation set, however for our purposes will serve as a way to measure performance on out-of-distribution corruptions. These are Gaussian blur, spatter, and saturate which correspond to the blur, weather, and digital categories respectively. We generate a validation set from the training set by randomly setting aside $1/50$ of the CIFAR10 training set and all of their corresponding corrupted variants. These experiments were run on a single GeForce RTX 2080 Ti graphics card, taking 22 hours to train the CVAE and 28 hours to run adversarial training. 

\begin{table}[t]
\parbox{.45\linewidth}{
  \caption{Prior and encoder architecture for learning CIFAR10 common corruptions}
  \label{table:cifar10encoder}
  \centering
  \begin{tabular}{cc}
    \toprule
    $\networkinput(3\times 32\times 32)$ & $\networkinput(k\times 32\times 32)$\\
    \cmidrule(r){1-1} \cmidrule(r){2-2}
    \multicolumn{2}{c}{$\concat(3,k)$}\\
    \midrule
    \multicolumn{2}{c}{$\conv(64)$}\\
    \midrule
    \multicolumn{2}{c}{$4 \times \residual(64)$}\\
    \midrule
    \multicolumn{2}{c}{$\conv(16)$}\\
    \cmidrule(r){1-1} \cmidrule(r){2-2}
    $\fc(512)$ & $\fc(512)$\\
    \bottomrule
  \end{tabular}
  }\hfill
  \parbox{.45\linewidth}{
  \caption{Decoder architecture for learning CIFAR10 common corruptions}
  \label{table:cifar10decoder}
  \centering
  \begin{tabular}{cc}
    \toprule
    $\networkinput(500)$ & \\
    \cmidrule(r){1-1} 
    $\fc(1\times 32\times 32)$ & $\networkinput(3\times 32\times 32)$\\
    \cmidrule(r){1-1} \cmidrule(r){2-2}
    \multicolumn{2}{c}{$\concat(1, 3)$}\\
    \midrule
    \multicolumn{2}{c}{$\conv(64)$}\\
    \midrule
    \multicolumn{2}{c}{$4 \times \residual(64)$}\\
    \midrule
    \multicolumn{2}{c}{$\conv(3)$}\\
    \bottomrule
  \end{tabular}
  }
\end{table}


\subsection{Perturbation model architecture and training specifics}
\label{app:cifar10c_perturbation_set}
We use standard preactivation residual blocks \citep{he2016identity}, with a convolutional bottleneck which reduces the number of channels rather than downsampling the feature space. This aids in learning to produce CIFAR10 common corruptions since the corruptions are reflective of local rather than global changes, and so there is not so much benefit from compressing the perturbation information into a smaller feature map. Specifically, our residual blocks uses convolutions that go from $64 \rightarrow 16 \rightarrow 64$ channels denoted as $\textrm{Residual}(64)$. Then, our encoder and prior networks are as shown in Table \ref{table:cifar10encoder}, 
where $k=0$ for the prior network and $k=3$ for the encoder network, and the decoder network is as shown in Table \ref{table:cifar10decoder}. 

\paragraph{Stabilizing the exponential function} Note that the CVAE encoders output the log variance of the prior and posterior distributions, which need to be exponentiated in order to calculate the KL distance. This runs the risk of of numerical overflow and exploding gradients, which can suddenly cause normal training to fail. To stabilize the training procedure, it is sufficient to use a scaled Tanh activation function before predicting the log variance, which is a Tanh activation which has been scaled to output a log variance betwen $[\ln(10^{-3}), \ln(10)]$. This has the effect of preventing the variance from being outside the range of $[10^{-3},10]$, and stops overflow from happening. In practice, the prior and posterior converge to a variance within this range, and so this doesn't seem to adversely effect the resulting representative power of the CVAE while stabilizing the exponential calculation. 

Training is done for 1000 epochs using a cyclic learning rate \citep{smith2017cyclical}, peaking at $0.001$ on the 400th epoch using the Adam optimizer \citep{kingma2014adam} with momentum $0.9$ and batch size $128$. The hyperparameter $\beta$ is also scheduled to increase linearly from $0$ to $10^{-2}$ over the first 400 epochs. We use standard CIFAR10 data augmentation with random cropping and flipping. 

\begin{table}[t]
  \caption{Measuring and comparing quality metrics for a CVAE perturbation set trained on CIFAR10 common corruptions with different pairing strategies depending on the available information.}
  \label{table:cifar10c_evaluate}
  \centering
  \begin{tabular}{lcrrrrrr}
    \toprule
    && \multicolumn{4}{c}{Test set quality metrics} & \multicolumn{2}{c}{Test set CVAE metrics}                  \\
    \cmidrule(r){3-6}\cmidrule(r){7-8}
    Method                    & $\epsilon$ & Enc. AE & PGD AE & EAE & OAE & Recon. err & KL \\
    \midrule
    Centered at original & $28$ & $0.005$ & $0.006$ & $0.029$ & $0.17$ & $5.73\cdot 10^{-4}$ & $69.31$ \\
    Original + perturbed & $34$ & $0.019$ & $0.010$ & $0.035$ & $0.20$ & $8.33\cdot 10^{-4}$ & $160.60$ \\
    Perturbed only       & $28$ & $0.019$ & $0.007$ & $0.038$ & $0.17$ & $8.53\cdot 10^{-4}$ & $182.85$ \\
    \bottomrule
  \end{tabular}
\end{table}
\subsection{Pairing strategies}
\label{app:cifar10_pairing}
Here, we quantitatively evaluate the effect of the different pairing strategies for the CIFAR10 common corruptions setting when training a perturbation set with a CVAE in Table \ref{table:cifar10c_evaluate}. We first note that the approach using random pairs of both original and perturbed data has a larger radius than the others, likely since the generator needs to produce not just the perturbed data but also the original as well. We next see that all metrics across the board are substantially lower for the approach which centers the prior by always conditioning on the unperturbed example, and so if such an unperturbed example is known, we recommend conditioning on said example to learn a better perturbation set. 

\begin{table}[t]
  \caption{CIFAR10 common corruptions validation set statistics for the $\ell_2$ norm of the latent space encodings.}
  \label{table:cifar10c_statistics}
  \centering
  \begin{tabular}{lcrrrrrr}
    \toprule
    Method & $\beta$ & Mean & Std & $25\%$  & $50\%$ & $75\%$ & Max \\
    \midrule
    Centered at original & $0.01$ & $6.53$ & $5.22$ & $2.69$ & $3.85$ & $10.16$ & $26.61$ \\
    Original + perturbed & $0.01$ & $11.50$ & $4.52$ & $7.81$ & $10.95$ & $14.48$ & $31.30$ \\
    Perturbed only & $0.01$ & $11.04$ & $3.81$ & $7.93$ & $10.82$ & $13.46$ & $29.57$ \\
    \bottomrule
  \end{tabular}
\end{table}

\subsection{Properties of the CVAE latent space}
\label{app:cifar10_latent}
We next study closely the properties of the underlying latent space which has been trained to generate CIFAR10 common corruptions. We summarize a number of statistics in Table \ref{table:cifar10c_statistics}, namely the mean, standard deviation, and some percentile thresholds where we can clearly see that the latent vectors for the CVAE centered around the original are much smaller in $\ell_2$ norm and thus learns the most compact perturbation set. We remind the reader that maximum $\ell_2$ norm reported here is calculated on the test set and is not necessarily the same as the radius reported in Table \ref{table:cifar10c_evaluate} since the radius is selected based on a validation set. 


\begin{figure}[t]
  \centering
  \begin{subfigure}{.3\textwidth}
  \centering
  \includegraphics[scale=1]{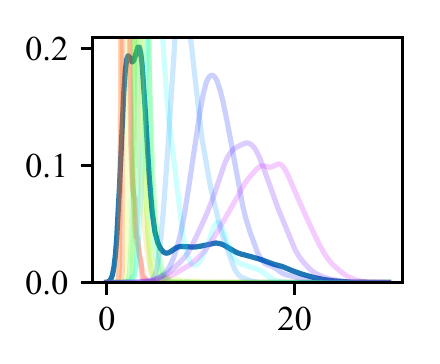}
  \caption{CVAE centered at original
  }
  \label{fig:cifar10c_latent_original}
\end{subfigure}
\begin{subfigure}{.3\textwidth}
  \centering
  \includegraphics[scale=1]{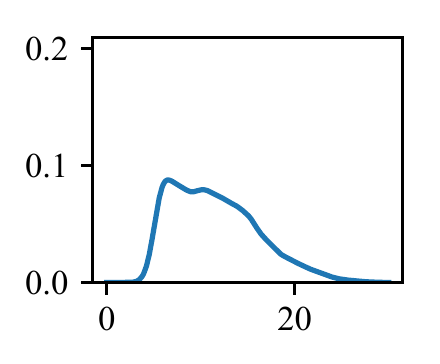}
  \caption{CVAE original + perturbed}
\end{subfigure}
\begin{subfigure}{.3\textwidth}
  \centering
  \includegraphics[scale=1]{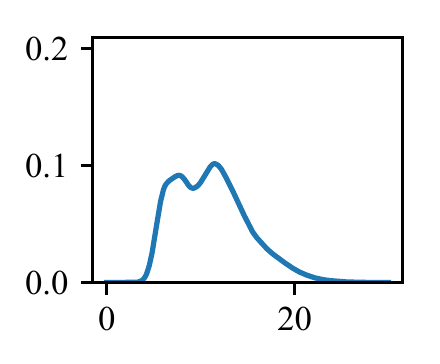}
  \caption{CVAE perturbed only}
\end{subfigure}
  \caption{Distribution of $\ell_2$ norms for latent encodings of CIFAR10 common corruptions on the test set for each pairing strategy.}
  \label{fig:cifar10c_latent}
\end{figure}

\paragraph{Distribution of $\ell_2$ distances in latent space} We plot the distribution of the $\ell_2$ norm of the latent space encodings in Figure \ref{fig:cifar10c_latent}, and further find the centered CVAE has semantic latent structure. Specifically, corruptions can be ordered by their $\ell_2$ norm in the latent space where each corruption inhabits a particular range of distance, so reducing the radius directly constrains the type of corruptions in the perturbation set. For example, restricting the perturbation set to radius $\epsilon=5$ corresponds to leaving out frost fog, snow, elastic, and glass blur corruptions, while containing all defocus blur, jpeg compression, and motion blur corruptions. A larger version of Figure \ref{fig:cifar10c_latent_original} with a complete legend listing all the corruptions is in Figure \ref{fig:density_large}. 

\begin{figure}[t]
  \centering
  \includegraphics[scale=1]{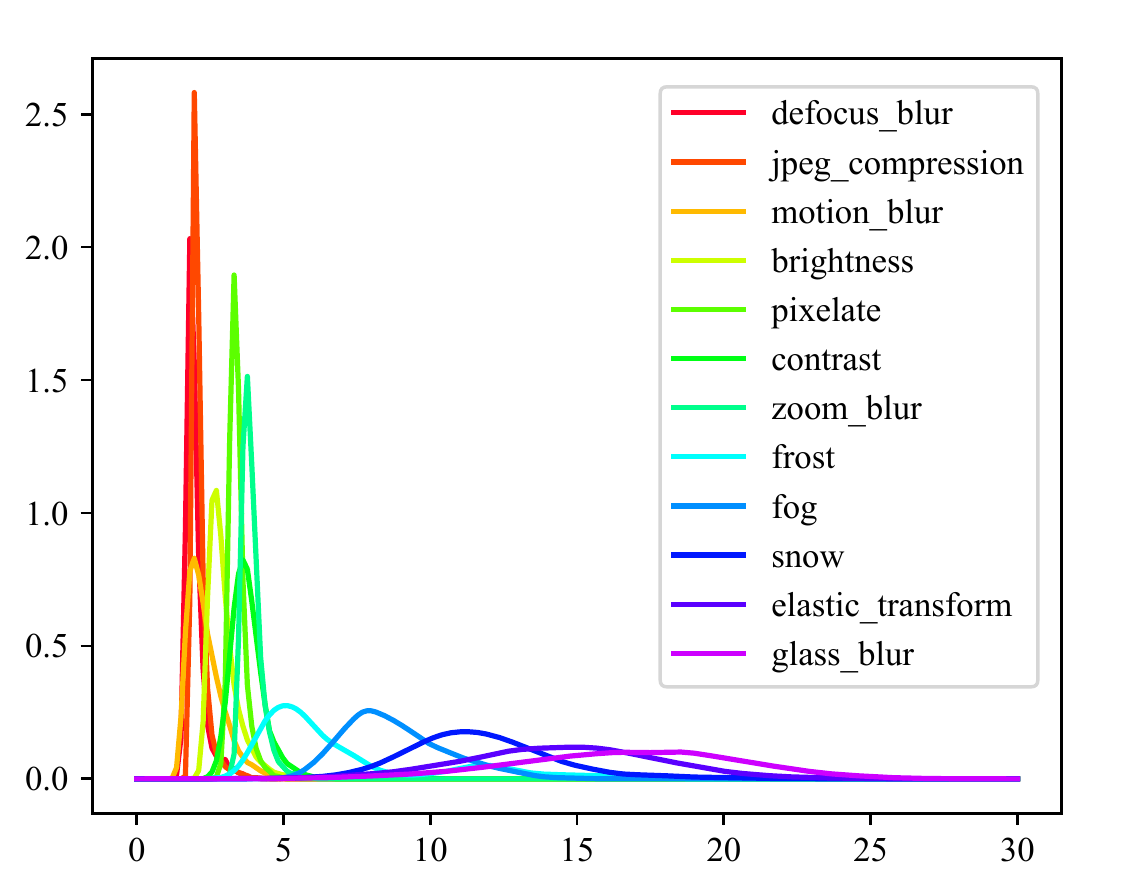}
  \caption{A larger version of the density of the $\ell_2$ norms of latent encodings of CIFAR10 common corruptions, broken down by type of corruption.}
  \label{fig:density_large}
\end{figure}

\begin{figure}[t]
  \centering
  \includegraphics[scale=1]{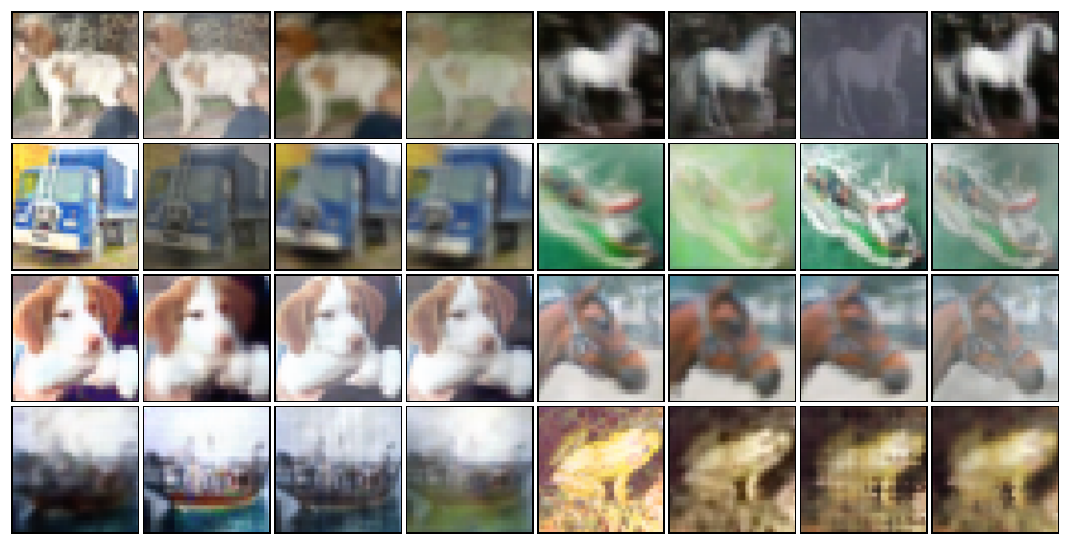}
  \caption{Random perturbations from the CVAE prior, showing four random samples for each example. }
  \label{fig:cifar10c_samples}
\end{figure}

\begin{figure}[t]
  \centering
  \includegraphics[scale=1]{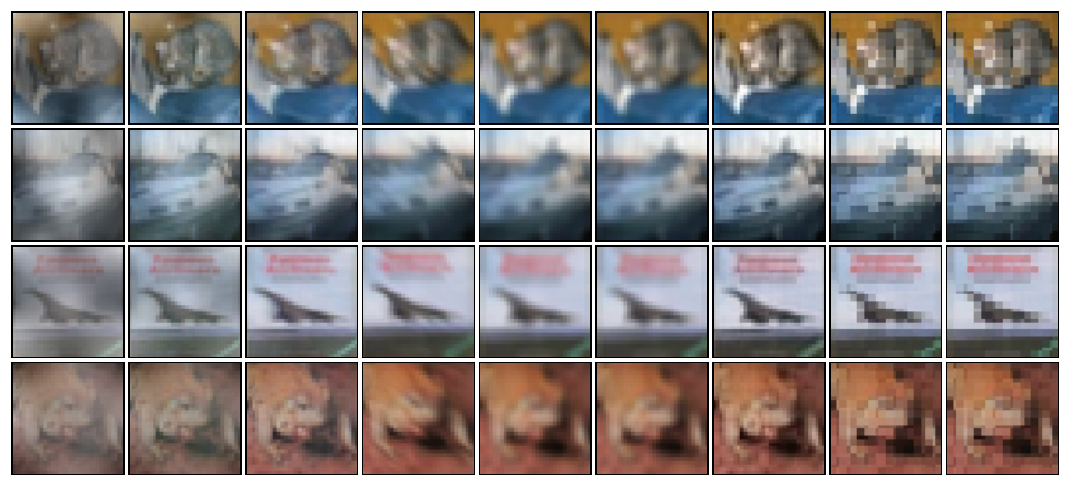}
  \caption{Interpolations between fog (left), defocus blur (middle), and pixelate (right) corruptions as representative distinct types of corruptions from the weather, blur, and digital corruption categories.}
  \label{fig:cifar10c_interpolations}
\end{figure}

\paragraph{Additional samples and interpolations from the CVAE}
To further illustrate what corruptions are represented in the CVAE latent space, we plot additional samples and interpolations from the CVAE. In Figure \ref{fig:cifar10c_samples} we draw four random samples for eight different examples, which show a range of corruptions and provide qualitative evidence that the perturbation set generates reasonable samples. In Figure \ref{fig:cifar10c_interpolations} we see a number of examples being interpolated between weather, blur, and digital corruptions, which also demonstrate how the perturbation set not only includes the corrupted datapoints, but also variations and interpolations in between corruptions.

\begin{figure}[t]
  \centering
  \includegraphics[scale=1]{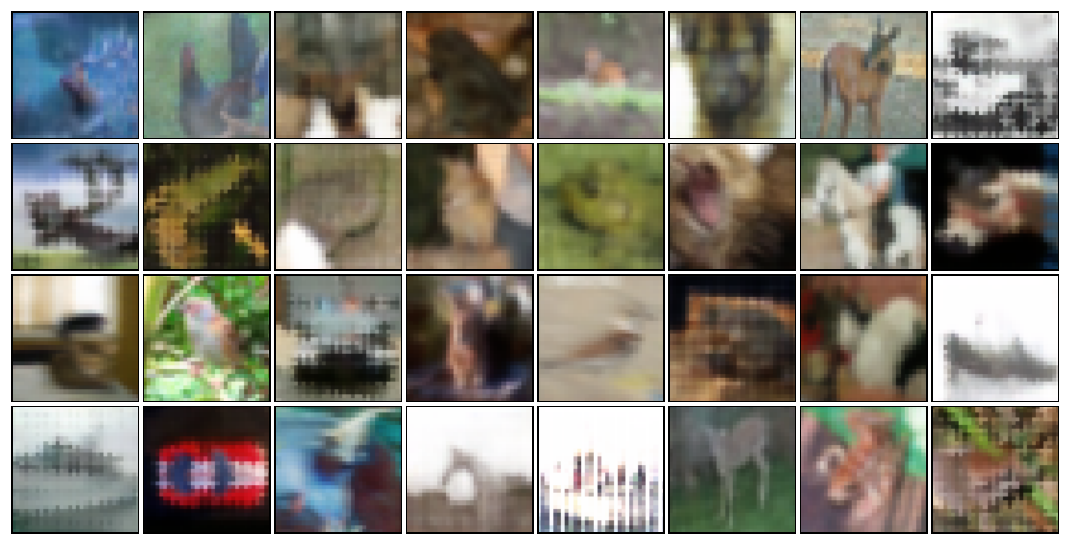}
  \caption{Adversarial examples that cause misclassification for an adversarially trained classifier.}
  \label{fig:cifar10c_adversarial}
\end{figure}

%


\subsection{Adversarial training}
\label{app:cifar10c_robust}
For learning a robust CIFAR10 classifier, we use a standard wide residual network \citep{zagoruyko2016wide} with depth 28 and width factor 10. All models are trained with the Adam optimizer with momentum 0.9 for 100 epochs with batch size 128 and cyclic learning rate schedule which peaks at 0.2 at epoch 40. This cyclic learning rate was tuned to optimize the validation performance for the data augmentation baseline, and kept fixed as-is for all other approaches. We use validation-based early stopping for all approaches to combat robust overfitting \citep{rice2020overfitting}. However, we find that in this setting and also likely due to the cyclic learning rate, we do not observe much overfitting. Consequently, the final model at the end of training ends up being selected for all methods regardless. Adversarial training is done with a PGD adversary at the full radius of $\epsilon=28$ with step size $\epsilon/5=5.6$ and 7 iterations. At evaluation time, for a given radius $\epsilon$ we use a PGD adversary with 50 steps of size $\epsilon/20$. Additional examples of adversarial attacks under this adversary are in Figure \ref{fig:cifar10c_adversarial}, which still appear to be reasonable corruptions despite being adversarial.  

\subsection{Comparisons to other baselines}
\label{app:baselines}
One may be curious as to how these results compare to other methods for learning models which are robust to \emph{different} threat models. In this section, we elaborate more on the comparison to other baselines. We preface this discussion with the following disclaimers: 
\begin{enumerate}
\item It \emph{not} necessarily expected for robustness to one threat model to generalize to another threat model. 
\item It is expected that training against a given threat model gives the best results to that threat model, in comparison to training against a different threat model.
\item These alternative baselines have different data assumptions (no access to corrupted data) or solve for a different type of robustness than what is being considered in this paper ($\ell_p$ robustness)
\end{enumerate}
Consequently, the fact that these baselines perform universally worse than the CVAE approaches is unsurprising and confirms our expectations. Nonetheless, we provide this comparison to satiate the readers curiosity. 

As alternative baselines, we compared to standard training, methods in general data augmentation (AugMix), $\ell_2$ adversarial robustness, and $\ell_\infty$ adversarial robustness for CIFAR10. For $\ell_2$ and $\ell_\infty$ adversarially robust models, we download pretrained state-of-the-art models ranked by the AutoAttack\footnote{\url{https://github.com/fra31/auto-attack}} leaderboard. For AugMix, there is no released CIFAR10 model, so we retrained an AugMix model using the official AugMix implementation\footnote{ \url{https://github.com/google-research/augmix}} using the parameters recommended by the authors. We note that the AugMix results for CIFAR10 in this paper are worse than what is reported by \citet{hendrycks2019augmix}, which is a known issue in the repository that others have encountered (see \url{https://github.com/google-research/augmix/issues/15}). 

\begin{table}[t]
  \caption{Certified robustness to CIFAR10 common corruptions with a CVAE perturbation set.}
  \label{table:cifar10c_certified}
  \centering
  \begin{tabular}{lrrrrrr}
    \toprule
    & \multicolumn{3}{c}{Test set accuracy $(\%)$} & \multicolumn{3}{c}{Test set certified accuracy $(\%)$}                   \\
    \cmidrule(r){2-4} \cmidrule(r){5-7}
    Noise level                    & Clean & Perturbed & OOD & $\epsilon=2.7$ & $\epsilon=3.9$ & $\epsilon=10.2$\\
    \midrule
    $\sigma=0.84$ & $93.8$ & $92.5$ & $84.2$ & $34.7$ & $00.0$ & $00.0$ \\
    $\sigma=1.22$ & $94.0$ & $92.6$ & $83.9$ & $66.1$ & $15.3$ & $00.0$ \\
    $\sigma=3.19$ & $86.7$ & $84.7$ & $65.0$ & $86.7$ & $39.4$ & $00.0$ \\
    \bottomrule
  \end{tabular}
\end{table}

\begin{figure}[t]
  \centering
  \includegraphics[scale=1]{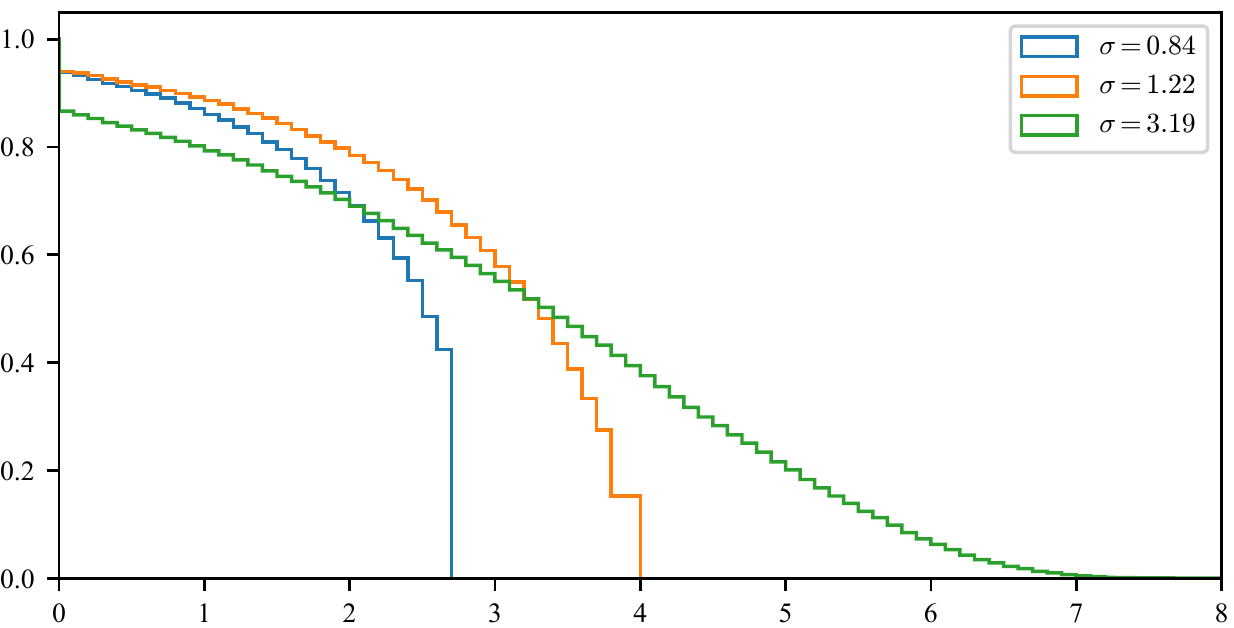}
  \caption{Certified accuracy against CIFAR10 common corruptions using randomized smoothing at various noise levels. The horizontal axis plots the certified radius, and the vertical axis plots the fraction of examples certified at that radius.}
  \label{fig:cifar10c_certified}
\end{figure}

\subsection{Randomized smoothing}
\label{app:cifar10c_certified}
Randomized smoothing is done in its most basic form, with normal Gaussian data augmentation at the specified noise level $\sigma$. We predict labels with $n_0=100$ samples and do certification with $n=10,000$ samples at confidence $\alpha = 0.001$, so each example has a $0.1\%$ chance of being incorrectly certified. Unlike in the typical setting, we have a known maximum radius for the perturbation set. This allows us to backwards engineer a noise level to have a specific maximum certified radius as follows: let $\ell$ be the maximum possible lower bound obtained from a confidence interval at level $1-\alpha$ with $n$ classifications that are in unanimous agreement. Then, the maximum certifiable radius is $r = \sigma \Phi^{-1}(\ell)$ where $\Phi^{-1}$ is the inverse cumulative distribution function of a standard normal distribution. By setting $r=\epsilon$, we can solve for the desired noise level $\sigma = \frac{r}{\Phi^{-1}(\ell)}$ to naturally select a noise level with maximum certified radius $\epsilon$. We do this for the $\epsilon \in \{ 2.7, 3.9, 10.2\}$ which are the $25$th, $50$th, and $75$th percentiles of the perturbation set, and get $\sigma \in \{ 0.84, 1.22, 3.19\}$ respectively. 


The results of randomized smoothing at these thresholds are presented in Table \ref{table:cifar10c_certified}. Note that by construction, each noise level cannot certify a radius larger than what is calculated for. We find that clean accuracy isn't significantly harmed when smoothing at the lower noise levels, and actually improves perturbed accuracy over the adversarial training approach from Table \ref{table:cifar10c_robustness} but with lower out-of-distribution accuracy than adversarial training. For example, smoothing with a noise level of $\sigma=0.84$ can achieve 92.5\% perturbed accuracy, which is $2\%$ higher than the best empirical approach from Table \ref{table:cifar10c_robustness}. Furthermore, all noise levels are able to achieve non-trivial degrees of certified robustness, with the highest noise level being the most robust but also taking a hit to standard test set accuracy metrics. Most notably, due to the latent structure present in the perturbation set as discussed in Appendix \ref{app:cifar10_latent}, even the ability to certify small radii translates to meaningful provable guarantees against certain types of corruptions, like defocus blur and jpeg compression which are largely captured at radius $\epsilon=3.9$. 

We note that although the largest noise level can theoretically reach the 75th percentile of the perturbation set at radius $\epsilon=10.2$, none of the examples can be certified to that extent. This is possibly a limitation of the certification procedure, since although a noise level of $\sigma=3.19$ can theoretically certify a radius of up to $\epsilon=10.2$, the $\ell_2$ norm of the random samples needed to perform randomized smoothing is approximately $23$, which is well beyond the 75th percentile and towards the boundary of the perturbation set. The complete robustness curves for all three noise levels are plotted in Figure \ref{fig:cifar10c_certified}. 

\section{Multi-illumination}
\label{app:mi}
In this final section we present the multi-illumination experiments in greater detail with an expanded discussion. 
We use the test set provided by \citet{murmann2019dataset} which consists of 30 held out scenes and hold out 25 additional ``drylab'' scenes for validation. 
Unlike in the CIFAR10 common corruptions setting, there is no such thing as an ``unperturbed'' example in this dataset so we train on random pairs selected from the 25 different lighting variations for each scene. These experiments were run on a single Quadro RTX 8000 graphics card, taking 16 hours to train the CVAE and 12 hours to run adversarial training. 

\subsection{Architecture and training specifics}
\label{app:mi_perturbation_set}
We convert a generic UNet architecture \citep{ronneberger2015u} to use as a CVAE, by inserting the variational latent space in between the skip connections of similar resolutions. At a high level, the encoder and prior networks will be based on the downsampling half of a UNet architecture, the conditional generator will be based on the full UNet architecture, and the components will be linked via the latent space of the VAE. Specifically, our networks have $[64, 128, 256, 512, 512]$ channels when downsampling, and twice the number of channels when upsampling due to the concatenation from the skip connection. 
Each skip connection passes through a $1 \times 1$ convolution that reduces the number of channels to 16, and an adaptive average pooling layer that reduces the feature maps to height and width $[(127,187), (62,93), (31,46), (15,23), (7,11)]$ respectively. The adaptive average pooling layer has the effect of a null operator for the MIP5 resolution, but allows higher resolutions to use the same architecture. These fixed-size feature maps are then passed through a fully connected layer to output a mean and log variance for the latent space with dimension $[128, 64, 32, 16, 16]$ for each respective feature map. The concatenation of all the latent vectors from each skip connection forms the full, $256$ dimensional latent space vector for the CVAE. Similar to the CIFAR10 setting, we use a scaled Tanh activation to stabilize the log variance calculation. 

The generator of the CVAE UNet is implemented as a typical UNet that takes as input the conditioned example, where intermediate feature maps are concatenated with extra feature maps from the latent space of the CVAE. Specifically, each latent space sub-vector is mapped with a fully connected layer back to a feature map with size $[(127,187), (62,93), (31,46), (15,23), (7,11)]$ with one channel. It is then interpolated to the actual feature map size of the UNet (which is a no-op for the MIP5 resolution) and concatenated to the standard UNet feature map before upsampling. 

We train the model for 1000 epochs with batch size 64, using the Adam optimizer with $0.9$ momentum, weight decay $5\cdot 10^{-7}$, and a cyclic learning rate from $[0.0001,0.001,0]$ over $[0,400,1000]$ epochs. Same as in the CIFAR10 setting, the learning rate and weight decay were chosen to optimize the validation performance of the baseline data augmentation approach and kept fixed as-is for all other methods. We weight the KL divergence with a $\beta$ hyperparameter which is scheduled from $[0,1,1]$ over $[0,400,1000]$, using random flip and crop with padding of $10$ for data augmentation. We then fine tune the model for higher resolutions: for MIP4 we fine tune for 100 epochs with the same learning rate schedule scaled down by a factor of 4, and for MIP3 we fine tune for 25 epochs with the learning rate schedule scaled down by 40. In order to keep the samples looking reasonable, for MIP3 we also increase the $\beta$ weight on the KL divergence to $10$. Random cropping augmentation is proportionately increased with the size of the image, using padding of 20 and 40 for MIP4 and MIP3 respectively.

\begin{table}[t]
  \caption{Measuring and comparing quality metrics for a multi-illumination CVAE at different resolutions.}
  \label{table:mi_evaluate}
  \centering
  \begin{tabular}{lccrrrrrr}
    \toprule
    &&\multicolumn{4}{c}{Test set quality metrics} & \multicolumn{2}{c}{Test set CVAE metrics}\\
    \cmidrule(r){3-6}\cmidrule(r){7-8}
    Resolution & $\epsilon$ &  Encoder AE & PGD AE & OAE  & EAE & Recon. err & KL \\
    \midrule
    MIP5 $(125\times 187)$  & $17$ & $0.019$ & $0.006$ & $0.049$ & $0.13$ & $0.0040$ & $65.8$ \\
    MIP4 $(250\times 375)$  & $25$ & $0.034$ & $0.008$ & $0.049$ & $0.27$ & $0.0042$ & $146.6$  \\
    MIP3 $(500\times 750)$  & $21$ & $0.060$ & $0.009$ & $0.048$ & $0.33$ & $0.0055$ & $106.1$\\
    \bottomrule
  \end{tabular}
\end{table}

\begin{table}[t]
  \caption{Multi-illumination validation set statistics for the $\ell_2$ norm of the latent space encodings.}
  \label{table:mi_statistics}
  \centering
  \begin{tabular}{lcrrrrrr}
    \toprule
    Resolution & $\beta$ & Mean & Std & $25\%$  & $50\%$ & $75\%$ & Max \\
    \midrule
    MIP5 $(125\times 187)$ & $1$ & $7.42$ & $2.14$ & $5.93$ & $7.35$ & $8.81$ & $16.63$ \\
    MIP4 $(250\times 375)$ & $1$ & $10.88$ & $3.18$ & $8.69$ & $10.75$ & $12.95$ & $24.69$ \\
    MIP3 $(500\times 750)$ & $10$ & $9.11$ & $2.72$ & $7.14$ & $9.00$ & $10.97$ & $20.65$ \\
    \bottomrule
  \end{tabular}
\end{table}

\paragraph{Learning and evaluating an illumination perturbation set at multiple scales}
The results of this fine tuning procedure to learn a perturbation at higher resolutions are summarized in Tables \ref{table:mi_evaluate} and \ref{table:mi_statistics}. As expected, the quality metrics of the perturbation set get slightly worse at higher resolutions which is counteracted to some degree by the increase in weighting for the KL divergence at MIP3. Despite using the same architecture size, the perturbation set is able to reasonably scale to images with 16 times more pixels and generate reasonable samples while keeping relatively similar quality metrics. However, to keep computation requirements at a reasonable threshold, we focus our experiments at the MIP5 resolution, which is sufficient for our robustness tasks.

\begin{figure}[t]
  \centering
  \includegraphics[scale=1]{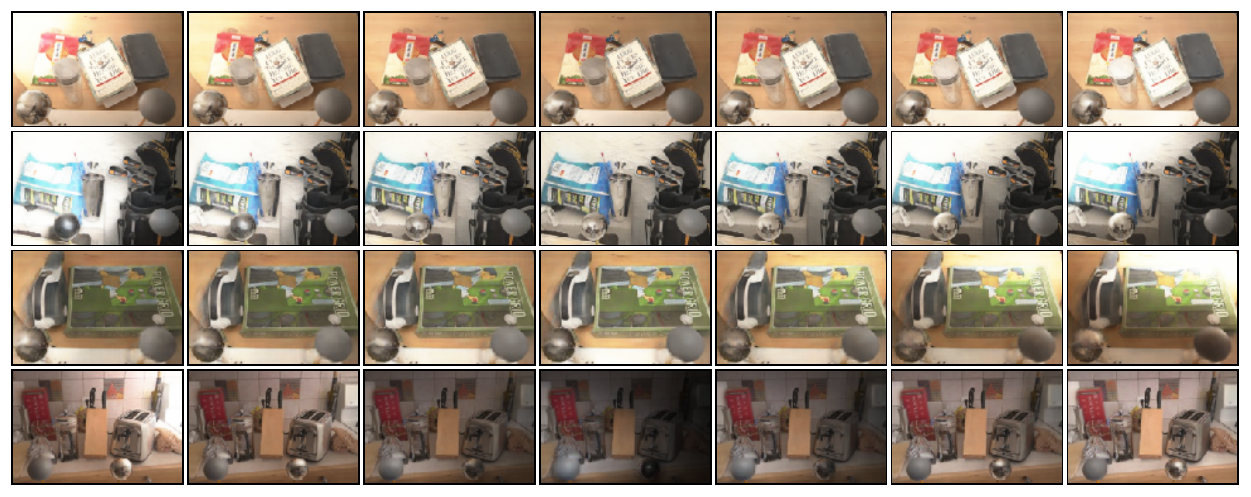}
  \caption{Additional interpolations between three (left, middle, right) randomly chosen lighting perturbations in each row.}
  \label{fig:mi_interpolation}
\end{figure}
 
\begin{figure}[t]
  \centering
  \includegraphics[scale=1]{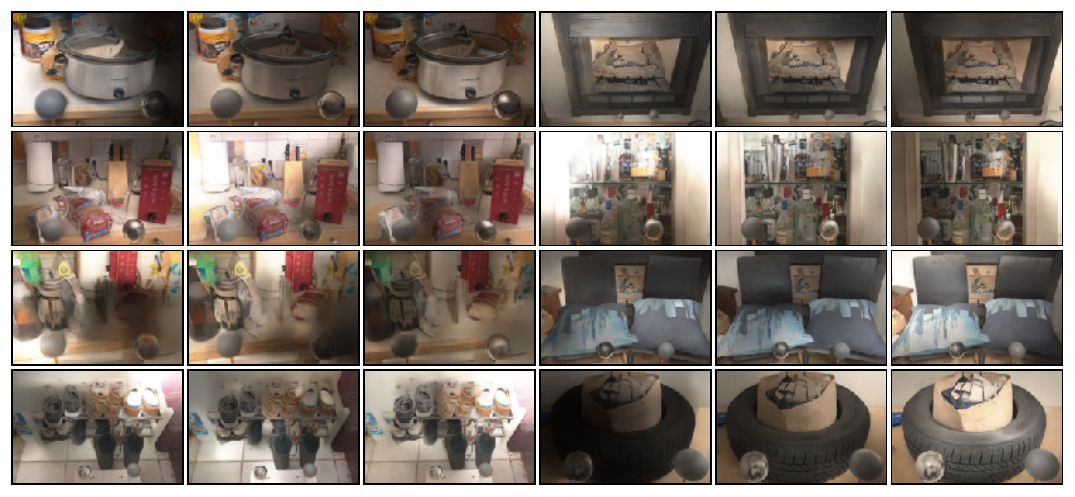}
  \caption{Additional random samples from the CVAE prior showing variety in lighting perturbations in various different scenes. }
  \label{fig:mi_samples}
\end{figure}

\subsection{Additional samples and interpolations from the CVAE}
\label{app:mi_visualizations}
We present additional samples and interpolations of lighting changes learned by the CVAE perturbation set. Figure \ref{fig:mi_interpolation} shows interpolations between three randomly chosen lighting perturbations for four different scenes, while Figure \ref{fig:mi_samples} shows 24 additional random samples from the perturbation set, showing a variety of lighting conditions. These demonstrate qualitatively that the perturbation set contains a reasonable set of lighting changes.

\begin{figure}[t]
  \centering
  \includegraphics[scale=1]{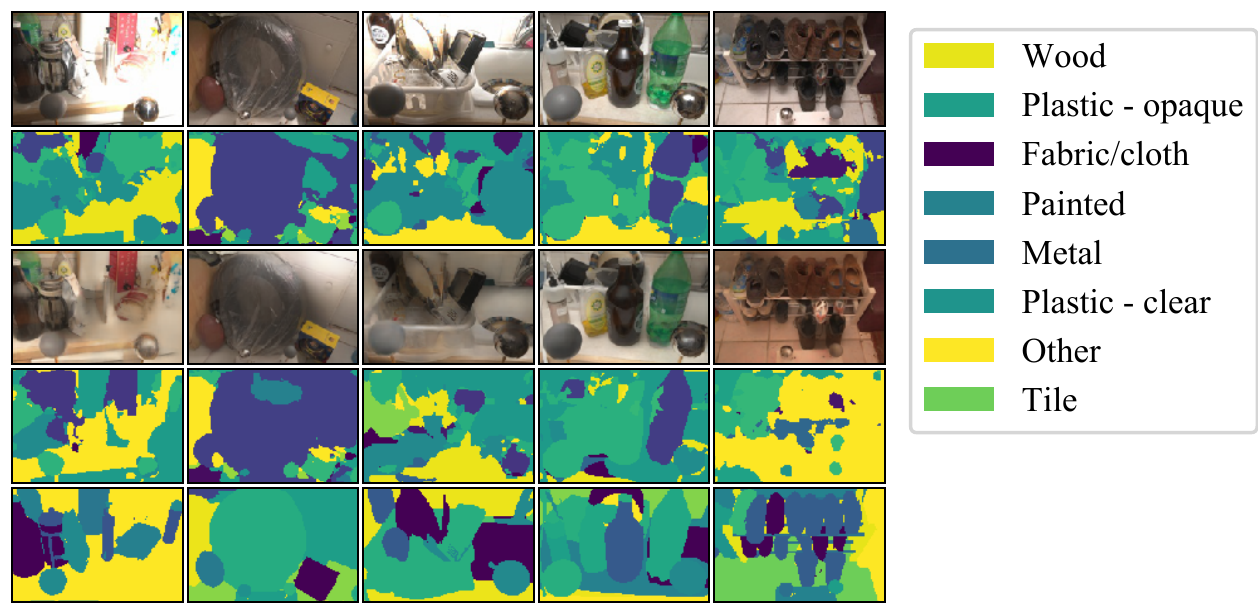}
  \caption{Adversarial examples that can cause on average 9.9\% more pixels in the shown segmentation maps to be incorrect for a model trained with data augmentation. The first two rows are a benign lighting perturbation and its corresponding predicted material segmentation, and the next two rows are an adversarial lighting perturbation and its corresponding predicted material segmentation. For reference, the final row contains the true material segmentation. }
  \label{fig:mi_adversarial_aug}
\end{figure}
 
\begin{figure}[t]
  \centering
  \includegraphics[scale=1]{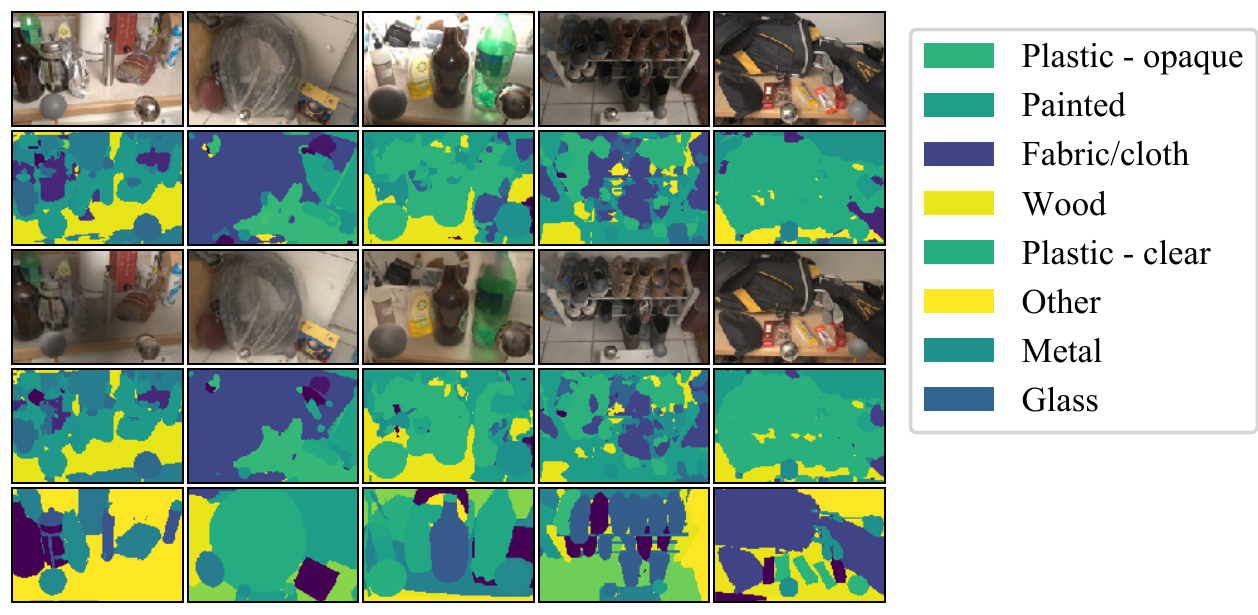}
  \caption{Adversarial examples that can cause on average 3\% more pixels in the shown segmentation maps to be incorrect for an adversarially trained model.}
  \label{fig:mi_adversarial}
\end{figure}

\begin{figure}[t]
  \centering
  \includegraphics[scale=1]{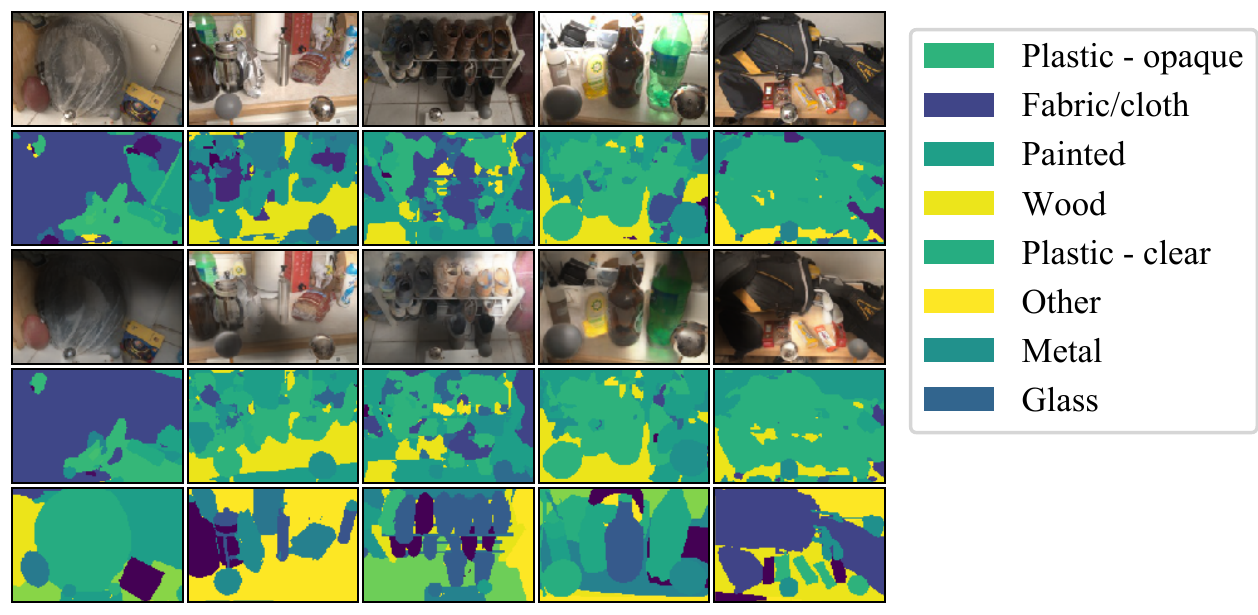}
  \caption{Adversarial examples at the full radius of $\epsilon=17$ that for an adversarially trained model, which are starting to cast dark shadows to obscure the objects in the image.}
  \label{fig:mi_adversarial_17}
\end{figure}


\subsection{Adversarial training and randomized smoothing}
\label{app:mi_robustness}
In this final task, we leverage our learned perturbation set to learn a model which is robust to lighting perturbations. Specifically the multi-illumination dataset comes with annotated material maps that label each pixel with the type of material (e.g. fabric, plastic, and paper), so a natural task to perform in this setting is to learn material segmentation maps. We report robustness results at $\epsilon \in \{ 7.35, 8.81, 17\}$ which correspond to the the 50th, 75th, and 100th percentiles. We select a fixed lighting angle which appears to be neutral and train a material segmentation model to generate a baseline, which achieves $37.2\%$ accuracy on the test set but only 14.9\% robust accuracy under the full perturbation model. 

We evaluate several methods for improving robustness of the segmentation maps to lighting perturbations, namely data augmentation with the perturbed data, data augmentation with the CVAE, and adversarial training with the CVAE. The results are tabulated in Table \ref{table:mi_robustness}. The CVAE data augmentation approach is not as effective in this setting at improving robust accuracy, as the pure data augmentation approach does reasonably well. However, the adversarial training approach unsurprisingly has the most robust accuracy, maintaining $35.4\%$ robust accuracy under the full perturbation set at $\epsilon=17$ and outperforming the data augmentation approaches. We plot adversarial perturbations at $\epsilon=7.35$ and the resulting changed segmentation maps in Figure \ref{fig:mi_adversarial_aug} for a model trained with pure data augmentation and in Figure \ref{fig:mi_adversarial} for a model trained to be adversarially robust, the latter of which has, on average, less pixels that are effected by an adversarial lighting perturbation. Adversarial examples at the full radius of $\epsilon=17$ are shown in Figure \ref{fig:mi_adversarial_17}, where we see that the perturbation set is beginning to cast dark shadows over regions of the image to force the model to fail.

\begin{figure}[t]
  \centering
  \includegraphics[scale=1]{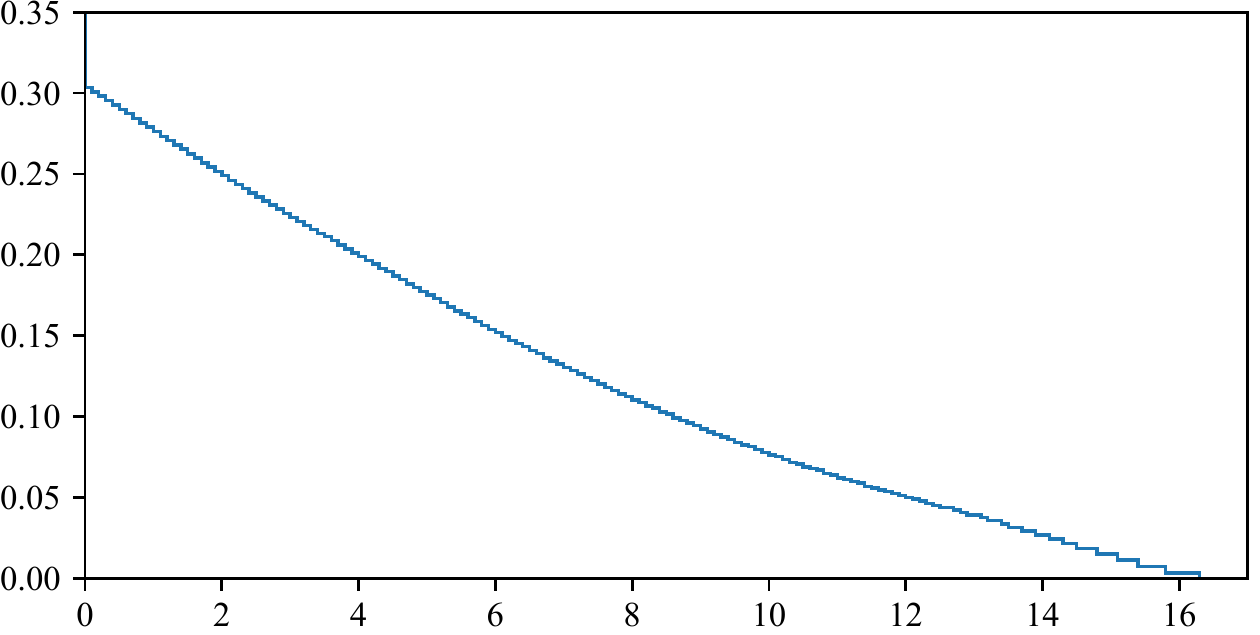}
  \caption{Certified accuracy for material segmentation model using randomized smoothing. The horizontal axis denotes the certified radius, and the vertical axis denotes the fraction of pixels that are certifiably correct at that radius. Note that a radius of $\epsilon=17$ is the maximum radius of the perturbation set. }
  \label{fig:mi_certified}
\end{figure}

Finally, we train a certifiably robust material segmentation model using the perturbation set. We train using a noise level of $\sigma=6.90$ which can certify a radius of at most $\epsilon=17$, or the limit of the perturbation set. The resulting robustness curve is plotted in Figure \ref{fig:mi_certified}. The model achieves $30.7\%$ perturbed accuracy and is able to get $12.4\%$ certified accuracy at the $50$th percentile of radius $\epsilon=7.35$. The key takeaway is that we can now certify real-world perturbations to some degree, in this case certifying robustness to 50\% of lighting perturbations with non-trivial guarantees. 
\end{document}